\documentclass[11pt]{article}
\usepackage{blindtext}
\usepackage[english]{babel}

\usepackage{amsmath}
\usepackage{amssymb}
\usepackage{amsthm}
\usepackage{geometry}
\usepackage[ruled, vlined]{algorithm2e} 

\usepackage{tikz}
\usetikzlibrary{calc, arrows.meta}

\newtheorem{lemma}{Lemma}
\newtheorem{example}{Example}
\newtheorem{definition}{Definition}
\newtheorem{proposition}{Proposition}
\newtheorem{remark}{Remark}
\newtheorem{corollary}{Corollary}
\newtheorem{theorem}{Theorem}
\usepackage{natbib}
\usepackage{hyperref}
\usepackage{cleveref}

\usepackage{subcaption}
\usepackage{pdflscape}

\begin{document}
\begin{center}
  {\LARGE Probabilistic Foundations of Fuzzy Simplicial Sets for Nonlinear Dimensionality Reduction}

  \vspace{1.5em}
   {\large Janis Keck$^{1,2,3,*}$\qquad 
 Lukas Silvester Barth$^{1,}$\qquad 
  Fatemeh (Hannaneh) Fahimi$^{1,4}$\qquad  \\ Parvaneh Joharinad$^{1,4}$
  J\"urgen Jost$^{1,2,5}$\qquad }

  \vspace{1.5em}

  {\normalsize
  $^1$Max Planck Institute for Mathematics in the Sciences, Leipzig, Germany\\
  $^2$Max Planck Institute for Human Cognitive and Brain Sciences, Leipzig, Germany\\
  $^3$Max Planck School of Cognition\\
  $^4$ScaDS.AI (Center for Scalable Data Analytics and AI), Leipzig, Germany\\
  $^5$Santa Fe Institute for the Sciences of Complexity, New Mexico, USA\\
    $^*$Corresponding author
  }

  \vspace{1.5em}

  \par
  \href{mailto:janis.keck@maxplanckschools.de}{\color{blue}{janis.keck@maxplanckschools.de}},\;
  \href{mailto:lukas.barth@mis.mpg.de}{\color{blue}{lukas.barth@mis.mpg.de}},\;
  \href{mailto:fatemeh.fahimi@mis.mpg.de}{\color{blue}{fatemeh.fahimi@mis.mpg.de}},\;
  \href{mailto:parvaneh.joharinad@mis.mpg.de}{\color{blue}{parvaneh.joharinad@mis.mpg.de}},\;
  \href{mailto:jjost@mis.mpg.de}{\color{blue}{jjost@mis.mpg.de}}
  
\end{center}

\begin{abstract}
Fuzzy simplicial sets have become an object of interest in dimensionality reduction and manifold learning, most prominently through their role in UMAP. However, their definition through tools from algebraic topology without a clear probabilistic interpretation detaches them from commonly used theoretical frameworks in those areas. In this work we introduce a framework that explains fuzzy simplicial sets as marginals of probability measures on simplicial sets. In particular, this perspective shows that the fuzzy weights of UMAP arise from a generative model that samples Vietoris–Rips filtrations at random scales, yielding cumulative distribution functions of pairwise distances. More generally, the framework connects fuzzy simplicial sets to probabilistic models on the face poset, clarifies the relation between Kullback–Leibler divergence and fuzzy cross-entropy in this setting, and recovers standard t-norms and t-conorms via Boolean operations on the underlying simplicial sets.
We then show how new embedding methods may be derived from this framework, and illustrate this on an example where we generalize UMAP using \v{C}ech filtrations with triplet sampling. In summary, this probabilistic viewpoint provides a unified probabilistic theoretical foundation for fuzzy simplicial sets, clarifies the role of UMAP within this framework, and enables the systematic derivation of new dimensionality reduction methods.
\end{abstract}

\section{Introduction}
Fuzzy simplicial sets have recently emerged as a theoretical concept fruitful for machine learning research, particularly in manifold learning, data visualization and clustering \citep{mcinnes2018umap, shiebler2020functorial,shiebler2021flattening}. Standard simplicial sets, long used in algebraic topology and topological data analysis, encapsulate the topological structure of a space in a combinatorial framework \citep{wasserman2018topological,friedman2012survey}. However, they do not inherently encode metric information, which is often crucial in data analysis. Fuzzy simplicial sets address this limitation by introducing membership strength functions that encode both combinatorial and metric properties \cite{spivak2009metric}.
The dimensionality reduction method UMAP has been of high success in leveraging this theoretical idea to obtain a guideline on how to effectively arrange points in low dimensional space as to capture this structure \citep{mcinnes2018umap,ghojogh2021uniform,diaz2021review,sainburg2021parametric}.
Despite UMAP's widespread adoption and some effort to investigate its properties \citep{damrich2022t,damrich2021umap,jardine2020stability,draganov2023actup,ravuri2024towards} the theoretical framework underlying fuzzy simplicial sets remains underappreciated in the broader machine learning and topological data analysis communities - possibly owed to the fact that fuzzy logic is less commonly used than the more familiar probability theory.
To unify these ideas, this work introduces a probabilistic perspective on fuzzy simplicial sets, framing them as objects generated by probability distributions over standard simplicial sets. By interpreting fuzzy weights as marginal probabilities, we not only provide an intuitive foundation for fuzzy simplicial sets but also establish that any such object can be generated probabilistically - at least in the finite setting relevant in practice.
Using this framework, we reinterpret the UMAP algorithm, demonstrating how its weights emerge from distributions over Vietoris-Rips complexes. This probabilistic view also suggests several avenues for generalizing UMAP, such as using alternative filtrations or optimizing over richer probabilistic models. We hope this perspective fosters a deeper understanding of fuzzy simplicial sets and their applications, paving the way for new methods in topological data analysis, dimensionality reduction or manifold learning.
This work thus makes three contributions: \textbf{(1)} We show that every finite fuzzy simplicial set arises as the image of a conventional probability distribution over standard simplicial sets, and we derive basic properties of this representation.
\textbf{(2)} We establish relationships to filtrations, t-norm and t-conorm operations, divergences, and comparisons between simplicial structures.
\textbf{(3)} Using the probabilistic formulation, we reinterpret UMAP and introduce a Čech-based variant that behaves similar as UMAP in preserving topological and geometric structure.

\section{Preliminaries}

We now collect the minimal definitions and concepts required for our formulation. These are standard, and we include them only for completeness and to fix notation, while also providing some examples and illustrations for readers not familiar with the definitions.

\subsection{Fuzzy Sets}
A fuzzy set is just a set where each element has an associated membership strength or weight:
\begin{definition}
  \label{def:classicalFuzzySet}
A classical fuzzy set is a set $S$ together with a weight function $\mu: S \to [0,1]$.
\end{definition}
Fuzzy sets have a rich history and well-developed theory, which we will not be able to even rudimentarily cover here - consider \cite{zimmermann2011fuzzy} for a full treatment.
Most importantly, they were conceived by \cite{zadeh1965fuzzy} to be able to model imprecise statements about membership (e.g., '$x$ is in the set of numbers \textit{much} larger than $y$') in a formal manner.
In contrast to probability theory, which may be interpreted as modelling uncertainty about outcomes, fuzzy set theory is concerned with modelling imprecise outcomes \citep{singpurwalla2004membership}. 
Due to these complementary objectives, since the inception of fuzzy set theory there have been made various attempts to combine these theories \citep{zadeh1968probability,hirota1981concepts,singpurwalla2004membership}, such that one may model imprecisely defined events in a probabilistic manner. To anticipate our discussion below: here, we will not need such a sophisticated approach, as we will merely provide a way in which the fuzzy objects we want to study are generated from probability measures.

As stated above, fuzzy sets are intended to model imprecise membership in a set. Many concepts from classical or 'crisp' set theory and the associated logic then generalize to fuzzy theory. 
We will only need the t-(co)-norms, which are generalizations of intersection and union operations, respectively. 
They are standard in the respective literatures, but we state the definition here for convenience of the reader:
\begin{definition}
A t-norm is a map 
\begin{equation}
T: [0,1] \times [0,1] \to [0,1]
\end{equation}
such that 
\begin{enumerate}
 \item  $T(a, b) = T(b, a)$ (Commutativity) 
 \item  $T(a, b) \leq T(c, d)$ if $a \leq c$ and $b \leq d$ (Monotonicity)
 \item $T(a, T(b, c)) = T(T(a, b), c)$ (Associativity)
  \item $T(a, 1) = a$ (Identity element).
\end{enumerate}
\end{definition}
To each t-norm $T$ we may associate a dual t-conorm $\tilde{T}$  via $\tilde{T}(a,b) = 1-T(1-a,1-b)$.
\begin{example}
Examples for a t-norm are the minimum
$$(a,b) \mapsto \min(a,b)$$
and the product norm 
$$(a,b) \mapsto ab.$$
Their dual t-conorms are the maximum 
$$(a,b) \mapsto \max(a,b)$$
and the probabilistic sum 
$$(a,b) \mapsto a + b - ab.$$
\end{example}
Naturally, one may identify classical or 'crisp' sets with those fuzzy sets that only take membership values in $\{0,1\}$ via $\mu \mapsto \mu^{-1}(1)$ (where $\mu$ is the weight function from Def.~\ref{def:classicalFuzzySet}) - we will use this identification repeatedly below.
One asserts that for classical sets the above operations indeed retrieve union and intersection of sets, furthermore, many of the properties of these operations carry over \citep{zadeh1965fuzzy}. 
\subsection{Fuzzy Simplicial Sets}
Our main object of interest are fuzzy \textit{simplicial sets}.  
Simplicial sets are a powerful tool in algebraic topology to encode topological information about a space in a combinatorial object of simplex/face relations. 
In \cite{spivak2009metric}, these were generalized to fuzzy simplicial sets, with the goal to also encode metric information. 
In brief, a fuzzy simplicial set is a simplicial set together with a fuzzy weight, were the fuzzy weight has to respect the additional structure imposed by the simplicial set. 
As we will see below, this will boil down to a certain monotonicity condition on the fuzzy weight, that is, there will be a partial order that the weights have to respect. 

Fuzzy simplicial sets were introduced in the language of category theory, which is the standard in algebraic topology. Here, we will eschew this language to make this text easier to follow for a general audience - an interested reader may consult \cite{barth2024fuzzysimplicialsetsapplication} for a thorough categorical treatment. 
When we want to model nested sets or simplex/face relations as in simplicial sets, then the fuzzy weight function has to be well-behaved with respect to the combinatorial structure:
\begin{definition}
A fuzzy simplicial set is a collection of fuzzy sets $(S_n,\mu_n), n \in \mathbb{N}$ together with 'face maps' $d_i^n: S_{n}  \to S_{n-1}$ and 'degeneracy maps' $s_i^n: S_n \to S_{n+1}$, which are both non-decreasing in the fuzzy weights, and which fulfill the simplicial identities:
\begin{equation} \begin{split}
d_{i}^{n-1}\,d_{j}^{n} 
&= 
d_{j-1}^{n-1}\,d_{i}^{n}
\quad (i<j),
\\[6pt]
s_{i}^{n+1}\,s_{j}^{n}
&=
s_{j+1}^{n+1}\,s_{i}^{n}
\quad (i \le j),
\\[6pt]
d_{i}^{n+1}\,s_{j}^{n}
&=
\begin{cases}
s_{j-1}^{n}\,d_{i}^{n}, & i< j,\\[4pt]
\mathrm{id}, & i=j\text{ or }i=j+1,\\[4pt]
s_{j}^{n}\,d_{i-1}^{n}, & i>j+1.
\end{cases}
\end{split} \end{equation}
\end{definition}
One then has as a special case:
\begin{definition}
A (classical/standard/crisp) simplicial set is a fuzzy simplicial set where all weight functions take values in $\{0,1\}$.
\end{definition}
Usually, simplicial sets are defined via set membership - for example, the reader might consult \cite{friedman2012survey} for an elementary introduction. One checks as before that this retrieves the usual definition by treating the weight function as an indicator function of set membership, that is by taking $\mu_n^{-1}(n)$ one obtains a collection of simplices as sets and face/degeneracy maps between them.
Again, we will use both notions interchangeably. 
\begin{remark}
Consider a standard simplicial set with sets $S_n$ and consider the union $S = \cup_n S_n$. We may introduce a partial order on $S$, where $\sigma \leq_{\text{face}} \sigma'$ iff there exists a sequence of face maps $d_{i_1},...,d_{i_k}$ such that $\sigma = d_{i_1}\circ ... \circ d_{i_k} (\sigma')$. Furthermore, we may introduce a second partial order on $S$, where $\sigma \leq_{\text{degeneracy}} \sigma'$ iff there exists a sequence of degeneracy maps $s_{i_1},...,s_{i_k}$ such that $\sigma = s_{i_1}\circ ... \circ s_{i_k}  (\sigma')$.
A classical fuzzy simplicial set may then equivalently be described (is isomorphic to) as a standard simplicial set together with a weight function $\mu: \cup_n S_n \to [0,1]$, where the weight function 
\begin{itemize}
\item is non-decreasing with respect to $\leq_{\text{face}}$ - this implies simplices have weight no bigger than their faces
\item is non-decreasing with respect to $\leq_{\text{degeneracy}}$ - this implies simplices have weight no bigger than their degeneracies.
\end{itemize}
Note that if  $\sigma'$  is a degeneracy of $\sigma$, i.e. for some $j$, $s_j(\sigma) = \sigma'$, then 
$\mu(\sigma') \leq \mu(\sigma)$ and $ \mu(\sigma) \leq \mu(\sigma')$, hence $\mu(\sigma') = \mu(\sigma)$.
\end{remark}
In practice, we often don't want to let $n$ go to arbitrary high values, but truncate it.
\begin{definition}
A truncated classical fuzzy simplicial set is obtained when truncating the indexing at a finite $n$.
\end{definition}
The sets $S_n$ which constitute a simplicial set in general may be arbitrary. We will however mostly be concerned with the case where we have some base set of vertices $X$ and then all higher order sets consist of unordered tuples from this base set. This is illustrated by the following example.
\begin{example}
Let $X$ be a set. The canonical free standard simplicial set
 generated by $X$, is defined as 
\begin{equation}
\begin{split}
S_0 &= X,\\ 
S_n &= X^{n+1} = X \times ... \times X =  \{[x_{i_0},...,x_{i_n}] \vert x_{i_j} \in X \}\\
d_k : [x_{i_0},...,x_{i_n}] &\mapsto [x_{i_0},... \hat{x_{i_k}},...x_{i_n}]\\
 s_k: [x_{i_0},...,x_{i_n} ] &\mapsto [x_{i_0},...x_{i_k},x_{i_k}, ...,,x_{i_n}]
 \end{split}
\end{equation}
where $[...]$ denotes ordered tuples  and $[...\hat x...]$ denotes a tuple with $x$ removed. \end{example}

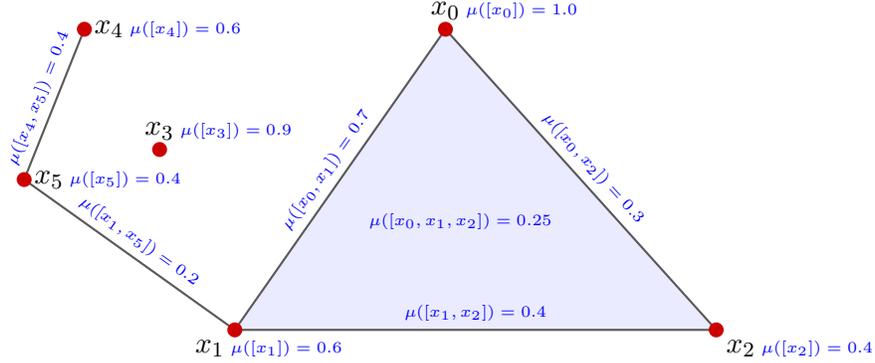
\begin{figure}
\centering
\begin{tikzpicture}[scale=2,
    vertex/.style={fill=red!80!black, circle, inner sep=2pt},
edge_weight_label/.style={midway,sloped, above, inner sep=1pt, font=\tiny, yshift=2pt,text=blue},
    vertex_weight_label/.style={ font=\tiny,text=blue}
]

    \coordinate (V0) at (0,1.5);
    \coordinate (V1) at (-1.4,-0.5);
    \coordinate (V2) at (1.8,-0.5);
    \coordinate (V3) at (-1.9,0.7);
    \coordinate (V4) at (-2.4,1.5);
    \coordinate (V5) at (-2.8,0.5);

    \fill[blue!10, opacity=0.8] (V0) -- (V1) -- (V2) -- cycle;

    \draw[thick, gray!70!black] (V0) -- (V1) node[edge_weight_label] {$\mu([x_0,x_1])=0.7$};
    \draw[thick, gray!70!black] (V1) -- (V2) node[edge_weight_label] {$\mu([x_1,x_2])=0.4$};
    \draw[thick, gray!70!black] (V2) -- (V0) node[edge_weight_label] {$\mu([x_0,x_2]) = 0.3$};
    \draw[thick, gray!70!black] (V1) -- (V5) node[edge_weight_label] {$\mu([x_1,x_5]) = 0.2$};
      \draw[thick, gray!70!black] (V5) -- (V4) node[edge_weight_label] {$\mu([x_4,x_5]) = 0.4$};

    \node[vertex] at (V0) {};
    \node[vertex] at (V1) {};
    \node[vertex] at (V2) {};
    \node[vertex] at (V3) {};
    \node[vertex] at (V4) {};
    \node[vertex] at (V5) {};

    \node (trianglelabel) at (0.1,0.1){};
    \node[above,vertex_weight_label] at (trianglelabel) {$\mu([x_0,x_1,x_2]) = 0.25$};

    \node[above] (V0label)  at (V0) {$x_0$};
    \node[below left]  (V1label) at (V1) {$x_1$};
    \node[below right]  (V2label) at (V2) {$x_2$};
    \node[above] (V3label) at (V3) {$x_3$};
    \node[right] (V4label) at (V4) {$x_4$};
    \node[right] (V5label) at (V5) {$x_5$};

    \node[vertex_weight_label, right,xshift = 4pt] at (V0label) {$\mu([x_0]) = 1.0$}; 
    \node[vertex_weight_label,right,xshift = 4pt] at (V1label) {$\mu([x_1]) = 0.6$};  
    \node[vertex_weight_label,right,xshift = 4pt] at (V2label) {$\mu([x_2]) = 0.4$}; 
    \node[vertex_weight_label,right,xshift = 4pt] at (V3label) {$\mu([x_3]) = 0.9$};
    \node[vertex_weight_label,right,xshift = 4pt] at (V4label) {$\mu([x_4]) = 0.6$};
    \node[vertex_weight_label,right,xshift = 4pt] at (V5label) {$\mu([x_5]) = 0.4$};
   
\end{tikzpicture}
\caption{Example of a finite, fuzzy simplicial set/complex. Simplices with zero weight are not plotted, as are degenerate simplices.}
\end{figure}

The ordered tuples will be called simplices.
As mentioned above, we are mainly concerned with the case where all simplices come from some base set of vertices. 
\begin{definition}\label{def:simpsetsoverX}
Let $X$ be a set. We denote by $\mathcal{S}(X)$ the collection of all simplicial sets with vertex set $X$, that is, $S_i = X^{i+1}$, and arbitrary weight function $\mu_i : X^{i+1} \to \{0,1\}$.
$\mathcal{S}^n(X)$ denotes respectively the collection of all such simplicial sets truncated at $n$.
Analogously, we define $\mathcal{F}(X), \mathcal{F}^n(X)$ for fuzzy simplicial sets.
\end{definition}
Since we are in practice only concerned with this case, in the following, to ease notation, we will often identify simplicial sets and their weights if it is clear from context what the underlying sets are. That is, we write then $S(\sigma)$ instead of $\mu(\sigma)$ for a simplex $\sigma$.
\subsection{Filtrations}

In geometrical and topological data analysis, one is often concerned with one-parameter filtrations of simplicial sets. For our purposes, those are simply constituted by a simplicial set $S$ and a family of weight functions $\mu^r$, indexed by some parameter $r$, such that the weight-function is monotonically increasing with respect to the parameter. This means that when increasing $r$, the strength of a simplex may not decrease, 
\begin{definition}
A filtration over $S$ is a collection of fuzzy simplicial sets $(S,\mu_n^r)$, where $r \leq t \implies \mu^r \leq \mu^t$.
\end{definition} 
In the special case of standard simplicial sets, this means once a simplex appears at a parameter $r$ it will be present for all further scales. 
\begin{remark}
Using the order on simplices we have stated before, a filtration alternatively is simply a map
\begin{equation}
\mu: S \times [0,1] \to [0,1],
\end{equation}
which is monotone in both of its arguments. 
\end{remark}

The most important example for us is the following:
\begin{definition}
Let $(X,d)$ be a metric space. The Vietoris-Rips filtration has sets $S_i = X^{i+1}$ and weights $\mu^r$ given by 
\begin{equation}
\mu^r_V\left([x_{i_0},...,x_{i_n}] \right) = \delta \left( \max_{j,k\in \{1,\cdots,n\}} d(x_{i_j},x_{i_k}) \leq r \right),
\end{equation}
where here and in the following $\delta(...)$ is the function returning $1$ if the statement inside the brackets is true and $0$ else.
We denote the individual simplicial sets as $VR(X,r):=(\mu^r_V)^{-1}(1)$.
\begin{figure}[h]
\centering
\begin{tikzpicture}[
    point_node/.style={circle, fill=red!60, draw=red!80!black, inner sep=2.5pt}, 
    edge_line/.style={draw=gray!70, line width=0.8pt}, 
    simplex_fill/.style={blue!10, opacity=0.8}, 
    epsilon_ball/.style={draw=orange!40, fill=orange!5, circle, opacity=0.6}, 
    scale=1.5
]

   
    \begin{scope}[scale=0.55,xshift=-4]
      \coordinate (P1) at (-0.5,0.5);
     \coordinate (P2) at (0,0);
    \coordinate (P3) at (1,0);
    \coordinate (P6) at (0.5,3.1);
    \coordinate (P4) at (2.7,0.5);
    \coordinate (P5) at (1.7,2.5);

        \node at (1, -2) {$VR(X; r_1)$};

        \draw[epsilon_ball] (P1) circle (0.8);
        \draw[epsilon_ball] (P2) circle (0.8);
        \draw[epsilon_ball] (P3) circle (0.8);
        \draw[epsilon_ball] (P4) circle (0.8);
        \draw[epsilon_ball] (P5) circle (0.8);
        \draw[epsilon_ball] (P6) circle (0.8);
        
        \draw[edge_line] (P1) -- (P2);

        \node at (P1) [point_node] {};
        \node at (P2) [point_node] {};
        \node at (P3) [point_node] {};
        \node at (P4) [point_node] {};
        \node at (P5) [point_node] {};
        \node at (P6) [point_node] {};

    \end{scope}

    \begin{scope}[scale=0.55,xshift=6.5cm] 
       \coordinate (P1) at (-0.5,0.5);
     \coordinate (P2) at (0,0);
    \coordinate (P3) at (1,0);
    \coordinate (P6) at (0.5,3.1);
    \coordinate (P4) at (2.7,0.5);
    \coordinate (P5) at (1.7,2.5);
                \node at (1, -2) {$VR(X; r_2)$};

        \draw[epsilon_ball] (P1) circle (1.5);
        \draw[epsilon_ball] (P2) circle (1.5);
        \draw[epsilon_ball] (P3) circle (1.5);
        \draw[epsilon_ball] (P4) circle (1.5);
        \draw[epsilon_ball] (P5) circle (1.5);
        \draw[epsilon_ball] (P6) circle (1.5);

   	\fill[simplex_fill] (P1) -- (P2) -- (P3) -- cycle;

 	 \draw[edge_line] (P1) -- (P2);
        \draw[edge_line] (P2) -- (P3);
        \draw[edge_line] (P1) -- (P3);
        \draw[edge_line] (P5) -- (P6);

        \node at (P1) [point_node] {};
        \node at (P2) [point_node] {};
        \node at (P3) [point_node] {};
        \node at (P4) [point_node] {};
        \node at (P5) [point_node] {};
        \node at (P6) [point_node] {};

    \end{scope}

    \begin{scope}[scale=0.55,xshift=13.5cm] 
        \coordinate (P1) at (-0.5,0.5);
     \coordinate (P2) at (0,0);
    \coordinate (P3) at (1,0);
    \coordinate (P6) at (0.5,3.1);
    \coordinate (P4) at (2.7,0.5);
    \coordinate (P5) at (1.7,2.5);    
    
                    \node at (1, -2) {$VR(X; r_3)$};

        \draw[epsilon_ball] (P1) circle (1.7);
        \draw[epsilon_ball] (P2) circle (1.7);
        \draw[epsilon_ball] (P3) circle (1.7);
        \draw[epsilon_ball] (P4) circle (1.7);
        \draw[epsilon_ball] (P5) circle (1.7);
        \draw[epsilon_ball] (P6) circle (1.7);

   	\fill[simplex_fill] (P1) -- (P2) -- (P3) -- cycle;

 	 \draw[edge_line] (P1) -- (P2);
        \draw[edge_line] (P2) -- (P3);
        \draw[edge_line] (P1) -- (P3);
        \draw[edge_line] (P5) -- (P6);
        \draw[edge_line] (P3) -- (P4);

        \node at (P1) [point_node] {};
        \node at (P2) [point_node] {};
        \node at (P3) [point_node] {};
        \node at (P4) [point_node] {};
        \node at (P5) [point_node] {};
        \node at (P6) [point_node] {};

    \end{scope}

\end{tikzpicture}
\caption{Visualization of the Vietoris-Rips-Filtration. With growing scale $r$, all simplices are added where the diameter (maximum distance between any two vertices) is less or equal than $r$.}
\end{figure}
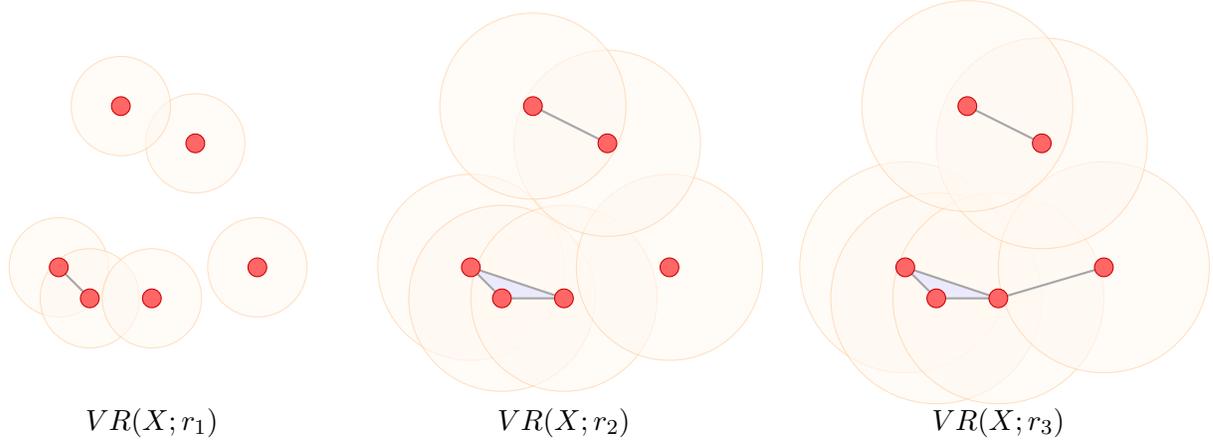
\end{definition}
From a theoretical point of view, the \v{C}ech-filtration  is also important. We give two definitions.
\begin{definition}
  \label{def:cechFiltration}
Let $(X,d)$ be a metric space. The intrinsic \v{C}ech-Filtration has sets $S_i = X^{i+1}$ and weights $\mu^r$ given by \begin{equation}
\mu^r_C\left([x_{i_0},...,x_{i_n}] \right)= \delta(\cap_k B_{r}(x_{i_k}) \neq \emptyset) =\delta(\inf_{y \in X } \max_{k} d(y,x_{i_k}) \leq r)
\end{equation}
If $X$ is itself a subspace of a bigger metric space $Y$, we may define the extrinsic \v{C}ech-Filtration, which has sets 
$S_i = X^{i+1}$ and weights $\mu^r$ given by \begin{equation}
\mu^r_C\left([x_{i_0},...,x_{i_n}] \right)= \delta(\cap_k B_{r}(x_{i_k}) \neq \emptyset) =\delta(\inf_{y \in Y } \max_{k} d(y,x_{i_k}) \leq r)
\end{equation}
\end{definition}

The (extrinsic) \v{C}ech Filtration is an important theoretical tool underlying topological data analysis (TDA) \citep{carlsson2021topological}, as we will briefly explain now.
We first briefly state the theoretical motivations, then we explain the intuition behind them.
Given a (paracompact) topological space $X$ and an open cover of that space $X = \cup_{i \in I} U_i$, the nerve of the cover is the simplicial set $\mathcal{N}(U)$, 
where a simplex $[i_1,...,i_k]$ is in the set iff the intersection $U_{i_1} \land ... \land U_{i_k}$ is nonempty.
The nerve theorem, which motivates TDA then tells us that in the case of a good open cover the geometric realization of $\mathcal{N}(U)$ is homotopy equivalent to $X$. 
Intuitively, this means that the former has the same topological features as the latter. This is of importance, as the former is a combinatorial object, while the latter is a space of possibly infinite size. Thus, the theorem gives a handle to encode topological information in a discrete object, the nerve.
Now, consider the situation where one has datapoints $x_1,...,x_n$ which one assumes are sampled from some unknown $M$ which is embedded in some euclidean space $\mathbb{R}^d$. Then, forming the extrinsic \v{C}ech filtration at scale $r$ on these points corresponds to constructing the nerve of the space $\cup_{i =1}^{n} B_r(x_i)$, that is the nerve of the open balls centered on the points. Thus, at each scale on has a topological space, and by the nerve theorem, the topological information of that space may be encoded in the simplicial set constructed on the vertices. Varying the scale and keeping track which features persist, one thus hopes to extract meaningful topological features of the underlying space $M$ - this is the conceptual underpinning of \textit{persistent homology} \citep{zomorodian2004computing}, where homology encodes topological features.
The stability and reconstruction theorems of TDA ensure that this endeavour is theoretically sound \citep{chazal2016structure}.

Now, in practice, computing the extrinsic \v{C}ech filtration at all scales is computationally expensive, especially for higher order simplices. 
Thus, in practice, often the \v{C}ech filtration is replaced by the Vietoris-Rips filtration. This is motivated by the fact that these filtrations are interleaved as
\begin{equation}
\mu^r_C  \leq \mu^r_V \leq \mu^{2r}_C.
\end{equation}

We note that under a minor requirement on the weight function, we may identify filtrations of classical simplicial sets and fuzzy simplicial sets:
Let $(S,\mu^r)$ be a filtration of classical simplicial sets, that is $\mu^r : S \to \{0,1\}$ monotone. 
Define a new weight 
\begin{equation}
\tilde{\mu}(\sigma) = \inf \{r: \mu^r(\sigma) = 1\}.
\end{equation}
Conversely, 
given a fuzzy weight $\tilde{\mu}$, define a filtration
\begin{equation}
\mu^r(\sigma) = \delta( \tilde{\mu}(\sigma) \geq r).
\end{equation}
One checks that these maps are inverses of each other, given one only admits fuzzy weights that are right-continuous in $r$.

\section{Fuzzy Simplicial Sets as Marginal Distributions}
Having established notation, we now introduce our main conceptual contribution: a probabilistic representation of fuzzy simplicial sets. In particular, we want to show how fuzzy weights naturally arise from probability distributions over standard simplicial sets. To make the definition of a map from distributions to fuzzy weights somewhat easier, we will need the following object:
\begin{definition}
For any $[x_{i_0},...,x_{i_n}] \in X^{n+1}$ we define $\mathbf{S}([x_{i_0},...,x_{i_n}] )$ to be the minimal simplicial set in $\mathcal{S}^n(X)$ containing $[x_{i_0},...,x_{i_n}]$ (that is where $\mu([x_{i_0},...,x_{i_n}]) = 1$). 
\end{definition}
The minimal simplicial set from the previous definition is obtained by taking all the faces of the given simplex, and then adding all necessary degeneracies, and then repeating this procedure until no new simplices are added.  In other words, 
\begin{lemma}
Let $[x_{i_0},...,x_{i_n}] \in X^{n+1}$ and $\sigma$ some arbitrary simplex. Let $\mu$ be the weight function of  $\mathbf{S}([x_{i_0},...,x_{i_n}] )$.
Then $\mu(\sigma) = 1$ if and only if there exist maps $f_1,f_2,...f_m$ such that 
\begin{equation}
\sigma = f_1 \circ f_2 \circ ... f_m ([x_{i_0},...,x_{i_n}] )
\end{equation} 
where all $f_j$ are face or degeneracy maps.
\end{lemma}
\begin{figure}
\begin{tikzpicture}[scale=1.7,
    vertex/.style={fill=red!80!black, circle, inner sep=2pt},
edge_weight_label/.style={midway,sloped, above, inner sep=1pt, font=\tiny, yshift=2pt,text=blue},
    vertex_weight_label/.style={ font=\tiny,text=blue}
]
    \coordinate (V0) at (0,1.5);
    \coordinate (V1) at (-1.4,-0.5);
    \coordinate (V2) at (1.8,-0.5);
    \coordinate (V3) at (-1.9,0.7);
    \coordinate (V4) at (-2.4,1.5);
    \coordinate (V5) at (-2.8,0.5);
    
    \coordinate (V0') at (4,1.5);
    \coordinate (V1') at (2.6,-0.5);
    \coordinate (V2') at (5.8,-0.5);

    \fill[blue!10, opacity=0.8] (V0) -- (V1) -- (V2) -- cycle;
    \fill[blue!10, opacity=0.8] (V0') -- (V1') -- (V2') -- cycle;

    \draw[thick, gray!70!black] (V0) -- (V1) ;
    \draw[thick, gray!70!black] (V1) -- (V2);
    \draw[thick, gray!70!black] (V2) -- (V0) ;
    \draw[thick, gray!70!black] (V1) -- (V5) ;
      \draw[thick, gray!70!black] (V5) -- (V4);

 \draw[thick, gray!70!black] (V0') -- (V1') ;
    \draw[thick, gray!70!black] (V1') -- (V2');
    \draw[thick, gray!70!black] (V2') -- (V0') ;

    \node[vertex] at (V0) {};
    \node[vertex] at (V1) {};
    \node[vertex] at (V2) {};
    \node[vertex] at (V3) {};
    \node[vertex] at (V4) {};
    \node[vertex] at (V5) {};
    \node[vertex] at (V0') {};
    \node[vertex] at (V1') {};
    \node[vertex] at (V2') {};
    \node[font=\large] at (2.2,0.5) {$\geq$};
        \node[font=\large] at (4,2.0) {$\mathbf{S}([x_0,x_1,x_2])$};

    \node(trianglelabel) at (0.1,0.1) {};

    \node[above] (V0label)  at (V0) {$x_0$};
    \node[below left]  (V1label) at (V1) {$x_1$};
    \node[below right]  (V2label) at (V2) {$x_2$};
    \node[above] (V3label) at (V3) {$x_3$};
    \node[right] (V4label) at (V4) {$x_4$};
    \node[right] (V5label) at (V5) {$x_5$};   
    
    \node[above] (V0'label)  at (V0') {$x_0$};
    \node[below left]  (V1'label) at (V1') {$x_1$};
    \node[below right]  (V2'label) at (V2') {$x_2$};
       
\end{tikzpicture}
\caption{Illustration of the minimal simplicial set for a given simplex. All simplices that are plotted are assumed to have weight $1$, all that are not plotted have weight $0$. The minimal simplicial set simply contains the simplex and all of its faces and degeneracies (the latter are not plotted).}
\end{figure}
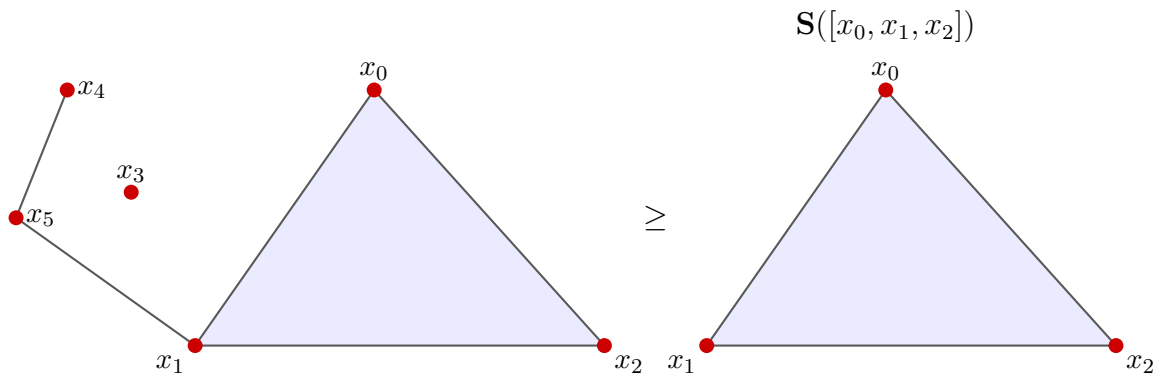

The weight function (or, alternatively, inclusion of simplices) induces a partial order on $\mathcal{S}^n(X)$, that is for two simplicial sets $S^1 = (X^{k+1}, k \in \{0,...,n\},\mu^1)$ and $ S^2 = (X^{k+1}, k \in \{0,...,n\},\mu^2)$ with the same underlying sets, we have $S^1 \geq S^2$ iff $\mu^1 \geq \mu^2$ everywhere.
Not all simplicial sets are comparable. However, for the minimal simplicial sets from the previous definition, the comparison is rather straightforward: One simply has to 
check whether the simplex defining the minimial simplicial set is present. That is:
\begin{lemma}
Let $S = (X^{k+1}, k \in \{0,...,n\},\mu)$ be a simplicial set.  Then 
\begin{equation}
S \geq \mathbf{S}([x_{i_0},...,x_{i_n}] ) \iff \mu([x_{i_0},...,x_{i_n}])= 1.
\end{equation}
\end{lemma}
\begin{remark}
If $X$ is a finite set, so is $\mathcal{S}^n(X)$. Thus, in that case we may always define a probability measure on $\mathcal{S}^n(X)$, that is simply a map $p: \mathcal{S}^n(X) \to [0,1]$ such that $\sum_{S \in S^n(X)} p(S) = 1$.
\end{remark}
As comparison against $\mathbf{S}$ is simply checking for the presence of the simplex, the according probability is the marginal probability of that simplex being present:
\begin{definition}
Let $X$ be a finite set. Let $p$ be a probability measure on $\mathcal{S}^n(X)$. We then define $m^n(p) : X^{n+1} \to [0,1]$
\begin{equation}
m^n(p)( [x_{i_0},...,x_{i_n}] ) = p\left[S \geq \mathbf{S}([x_{i_0},...,x_{i_n}]) \right] = p(S([x_{i_0},...,x_{i_n}])=1)
\end{equation}
where $S$ denotes a simplicial set randomly sampled from $p$, which is identified which its weight function on the rightmost side.
\end{definition}
The second equality in this definition follows from our previous discussion on comparisons with the minimal simplicial sets. 
Thus, one may interpret $m^n$ as the marginal probability of observing a given simplex (i.e. having weight one) in a randomly sampled simplical set under $p$.  
The class of probability distributions that we want to consider are the following:
\begin{definition}
Let $X$ be a set, and consider the truncated simplicial sets $\mathcal{S}^n(X)$.  We define a family of probability measures consistent with the simplicial structure to be a family of probability measures 
\begin{equation} p_{x_{i_0},...,x_{i_n}} : \mathcal{S}^n(\{x_{i_0},...,x_{i_n}\}) \to [0,1]\end{equation}
where the $x_{i_k} \in X$,
such that whenever there is an intersection of the points of two of those measures, the $m(p)$ agree on the shared simplices. 
That is, if $U = \{x_{i_0},...,x_{i_n}\} \cap \{y_{i_1},..., y_{i_n}\}$, then 
\begin{equation}
m(p_{x_{i_0},...,x_{i_n}})([u_{i_1},..,u_{i_k}]) = m(p_{y_{i_1},...,y_{i_n}})([u_{i_1},..,u_{i_k}])  \forall u_{i_j} \in U.
\end{equation}
\end{definition}
As an example (this foreshadows an example below), one may first think of a probability measure that generates independent, identically distributed points $x_i$ in some euclidean space, together with a deterministic rule on how to construct simplices with these points as vertices - e.g., the Vietoris-Rips complex at a certain scale.
One then checks that this gives a consistent family.
We are now able to state a straightforward result linking probability distributions and fuzzy weights:
\begin{proposition}
Given a family of probability measures consistent with the simplicial structure on $X$, $(S_k,\mu^k)$ defines a classical fuzzy simplicial set, where $S_k = X^{k+1}$ and
\begin{equation}
\mu^k([x_{i_0},...,x_{i_k}]) = m(p_{z_{i_0},...,z_{i_n}}) ([x_{i_0},...,x_{i_k}])
\end{equation}
where $z_{i_0},...,z_{i_n}$ are any such points containing $x_{i_0},...,x_{i_k}$.\end{proposition}
\begin{proof}
By the consistency, the weight $\mu$, is well defined, that is it will not depend on our choice of base points. 
We then only have to check that the weight is compatible with face and degeneracy maps.
This follows from the definition:
\begin{equation} \begin{split}
\mu(d_j[x_{i_0},..,x_{i_k}]) &=  p(S \geq \mathbf{S}([x_{i_0},..., \hat{x_{i_j}},..., x_{i_k}]) \\ &\geq p(S \geq \mathbf{S}([x_{i_0},,..., x_{i_k}]) = \mu([x_{i_0},...,x_{i_k}]),
\end{split} \end{equation}
since $\mathbf{S}([x_{i_0},..., \hat{x_{i_j}},..., x_{i_k}]) \leq \mathbf{S}([x_{i_0},,..., x_{i_k}])$ - the minimal simplicial set that contains the face of a simplex is contained in the minimial simplicial set of that simplex.
On the other hand,
\begin{equation} \begin{split}
\mu(s_j[x_{i_0},..,x_{i_k}]) &=  p(S \geq \mathbf{S}([x_{i_0},...,{x_{i_j}}, x_{i_j},..., x_{i_k}])\\ &= p(S \geq \mathbf{S}([x_{i_0},,..., x_{i_k}]) = \mu([x_{i_0},...,x_{i_k}]).
\end{split} \end{equation}
This is because the existence of a simplex in a simplicial set (i.e. having weight $1$) necessitates all degeneracies of that simplex to be also present, and the presence of a degeneracy necessitates all faces, in particular non-denegerate ones, which implies that the minimal simplicial sets agree. 
Thus we have shown that both face and degeneracy maps do not decrease weight.
\end{proof}
\begin{remark}
The above proposition essentially hinges on the underlying poset structure of simplices and may be generalized to arbitrary posets where these results are standard (see the appendix).
\end{remark}
\begin{figure}
\centering
\begin{tikzpicture}[  vertex/.style={fill=red!80!black, circle, inner sep=2pt,font=\tiny},
edge_weight_label/.style={midway,sloped, above, inner sep=1pt, font=\tiny, yshift=2pt,text=blue},
    vertex_weight_label/.style={ font=\tiny,text=blue}]

    
            \begin{scope}[
            scale=0.7,  ]
          \coordinate (V0) at (0,1.5);
    \coordinate (V1) at (-1.4,-0.5);
    \coordinate (V2) at (1.8,-0.5);
    \coordinate (V3) at (-1.9,0.7);
    \coordinate (V4) at (-2.4,1.5);
    \coordinate (V5) at (-2.8,0.5); 
     \node (p) at (0,3) {$p = \frac{2}{8}$}; 

         \fill[blue!10, opacity=0.8] (V0) -- (V1) -- (V2) -- cycle;

	 \draw[thick, gray!70!black] (V0) -- (V1) ;
    \draw[thick, gray!70!black] (V1) -- (V2);
     \draw[thick, gray!70!black] (V2) -- (V0) ;
      \draw[thick, gray!70!black] (V0) -- (V3) ;
         \draw[thick, gray!70!black] (V1) -- (V3);
    \draw[thick, gray!70!black] (V3) -- (V4);
     \draw[thick, gray!70!black] (V3) -- (V5) ;

    \node[vertex] at (V0) {};
    \node[vertex] at (V1) {};
    \node[vertex] at (V2) {};
    \node[vertex] at (V3) {};
    \node[vertex] at (V4) {};
    \node[vertex] at (V5) {};

    \node[above, font=\tiny] (V0label)  at (V0) {$x_0$};
    \node[below left, font=\tiny]  (V1label) at (V1) {$x_1$};
    \node[below right, font=\tiny]  (V2label) at (V2) {$x_2$};
    \node[above, font=\tiny] (V3label) at (V3) {$x_3$};
    \node[right, font=\tiny] (V4label) at (V4) {$x_4$};
    \node[right, font=\tiny] (V5label) at (V5) {$x_5$}; 
       \end{scope}  
       

        \begin{scope}[
            scale=0.7, shift={(7,0)} ]
          \coordinate (V0') at (0,1.5);
    \coordinate (V1') at (-1.4,-0.5);
    \coordinate (V2') at (1.8,-0.5);
    \coordinate (V3') at (-1.9,0.7);
    \coordinate (V4') at (-2.4,1.5);
    \coordinate (V5') at (-2.8,0.5); 
       \node (p) at (0,3) {$p = \frac{3}{8}$};

    \fill[blue!10, opacity=0.8] (V3') -- (V4') -- (V5') -- cycle node[midway] (V3V4V5') {};
    \fill[blue!10, opacity=0.8] (V0') -- (V1') -- (V2') -- cycle ;
 \draw[thick, gray!70!black] (V0') -- (V1') ;
    \draw[thick, gray!70!black] (V1') -- (V2');
     \draw[thick, gray!70!black] (V2') -- (V0') ;

 \draw[thick, gray!70!black] (V3') -- (V4') ;
    \draw[thick, gray!70!black] (V4') -- (V5');
     \draw[thick, gray!70!black] (V5') -- (V3') ;

     \node[vertex] at (V0') {};
    \node[vertex] at (V1') {};
    \node[vertex] at (V2') {};
    \node[vertex] at (V3') {};
    \node[vertex] at (V4') {};
    \node[vertex] at (V5') {};
     \node[above, font=\tiny] (V0label)  at (V0') {$x_0$};
    \node[below left, font=\tiny]  (V1label) at (V1') {$x_1$};
    \node[below right, font=\tiny]  (V2label) at (V2') {$x_2$};
    \node[above, font=\tiny] (V3label) at (V3') {$x_3$};
    \node[right, font=\tiny] (V4label) at (V4') {$x_4$};
    \node[right, font=\tiny] (V5label) at (V5') {$x_5$}; 
       \end{scope}

       \begin{scope}[
            scale=0.7, shift={(14,0)} ]
          \coordinate (V0'') at (0,1.5);
    \coordinate (V1'') at (-1.4,-0.5);
    \coordinate (V2'') at (1.8,-0.5);
    \coordinate (V3'') at (-1.9,0.7);
    \coordinate (V4'') at (-2.4,1.5);
    \coordinate (V5'') at (-2.8,0.5); 
           \node (p) at (0,3) {$p = \frac{2}{8}$};

    \fill[blue!10, opacity=0.8] (V0'') -- (V3'') -- (V4'') -- cycle;
    \fill[blue!10, opacity=0.8] (V3'') -- (V4'') -- (V5'') -- cycle node[midway] (V3V4V5'') {};
    \fill[blue!10, opacity=0.8] (V1'') -- (V2'') -- (V3'') -- cycle;

 \draw[thick, gray!70!black] (V0'') -- (V3'') ;
  \draw[thick, gray!70!black] (V3'') -- (V4'') ;
  \draw[thick, gray!70!black] (V4'') -- (V0'') ;
  \draw[thick, gray!70!black] (V4'') -- (V5'') ;

  \draw[thick, gray!70!black] (V5'') -- (V3'') ;
    \draw[thick, gray!70!black] (V3'') -- (V2'');
    \draw[thick, gray!70!black] (V2'') -- (V1'');
    \draw[thick, gray!70!black] (V1'') -- (V3'');

     \node[vertex] at (V0'') {};
    \node[vertex] at (V1'') {};
    \node[vertex] at (V2'') {};
    \node[vertex] at (V3'') {};
    \node[vertex] at (V4'') {};
    \node[vertex] at (V5'') {};
     \node[above, font=\tiny] (V0label)  at (V0'') {$x_0$};
    \node[below left, font=\tiny]  (V1label) at (V1'') {$x_1$};
    \node[below right, font=\tiny]  (V2label) at (V2'') {$x_2$};
    \node[above, font=\tiny] (V3label) at (V3'') {$x_3$};
    \node[right, font=\tiny] (V4label) at (V4'') {$x_4$};
    \node[right, font=\tiny] (V5label) at (V5'') {$x_5$}; 
       \end{scope}
       \begin{scope}[
            scale=0.7, shift={(21,0)} ]
          \coordinate (V0''') at (0,1.5);
    \coordinate (V1''') at (-1.4,-0.5);
               \node (p) at (0,3) {$p = \frac{1}{8}$}; 

     \draw[thick, gray!70!black] (V0''') -- (V1''')  node[midway] (V0V1''') {};

        \node[vertex] at (V0''') {};
    \node[vertex] at (V1''') {};
   
        \node[above, font=\tiny] (V0label)  at (V0''') {$x_0$};
    \node[below left, font=\tiny]  (V1label) at (V1''') {$x_1$};
          \end{scope}
          
            \begin{scope}[
            scale=2.0, shift={(4,-4)} ]
          \coordinate (V0'''') at (0,1.5);
    \coordinate (V1'''') at (-1.4,-0.5);
    \coordinate (V2'''') at (1.8,-0.5);
    \coordinate (V3'''') at (-1.9,0.7);
    \coordinate (V4'''') at (-2.4,1.5);
    \coordinate (V5'''') at (-2.8,0.5); 
     
     \fill[blue!10, opacity=0.8] (V3'''') -- (V4'''') -- (V5'''') -- cycle node[midway] (V3V4V5'''') {};

    \draw[thick, gray!70!black] (V0'''') -- (V1'''')  node[midway] (V0V1'''') {};

     \node[vertex] at (V0'''') {};
    \node[vertex] at (V1'''') {};
    \node[vertex] at (V2'''') {};
    \node[vertex] at (V3'''') {};
    \node[vertex] at (V4'''') {};
    \node[vertex] at (V5'''') {};
     \node[above, font=\tiny] (V0label)  at (V0'''') {$x_0$};
    \node[below left, font=\tiny]  (V1label) at (V1'''') {$x_1$};
    \node[below right, font=\tiny]  (V2label) at (V2'''') {$x_2$};
    \node[above, font=\tiny] (V3label) at (V3'''') {$x_3$};
    \node[right, font=\tiny] (V4label) at (V4'''') {$x_4$};
    \node[right, font=\tiny] (V5label) at (V5'''') {$x_5$}; 
       \end{scope}
            \draw[dashed, green] (V2) -- (V2'''');
             \draw[dashed, green] (V2') -- (V2'''');
            \draw[dashed, green] (V2'') -- (V2'''');

           \draw[dashed, blue] (V0V1''') -- (V0V1'''');
            \draw[dashed, red] (V3V4V5') -- (V3V4V5'''');
            \draw[dashed, red] (V3V4V5'') -- (V3V4V5'''');
          \node[above left, font=\small,red] at (V3V4V5'''') {$ \mu([x_3,x_4,x_5]) = \frac{5}{8}$};
          \node[right,font=\small,blue] at (V0V1'''') {$ \mu([x_1,x_2]) = \frac{1}{8}$};
           \node[left,font=\small,green] at (V2'''') {$ \mu([x_2]) = \frac{7}{8}$};
           
\end{tikzpicture}
\caption{Illustration of the procedure for obtaining fuzzy weights from a probability distribution over a simplicial set. Top row shows a probability distribution over 4 simplicial sets (only the nondegenerate simplices with weight $1$ are shown). The bottom plot shows how the fuzzy weights for some of the simplices are obtained by computing the marginal probability of observing them in any of the simplicial sets.}
\end{figure}
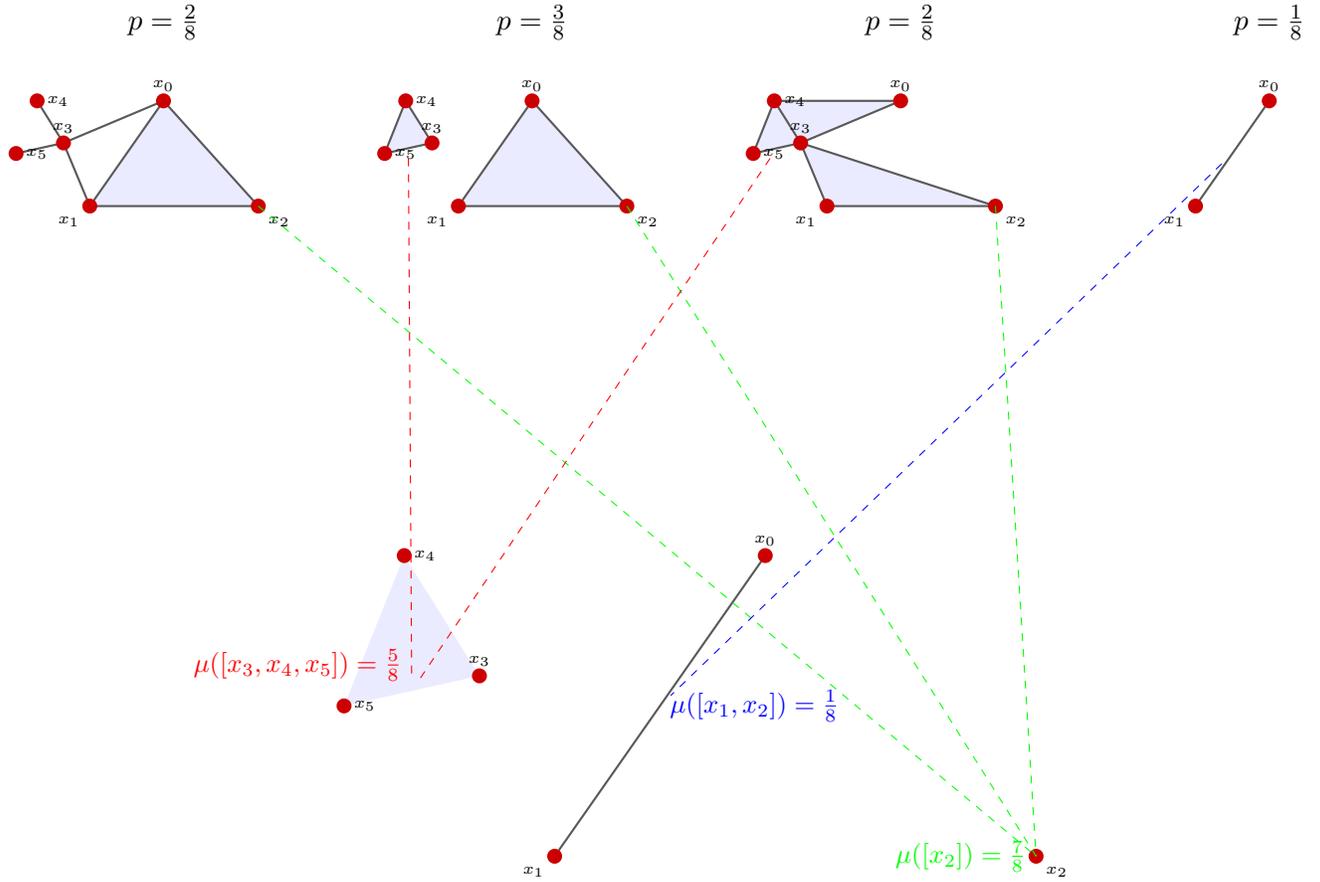

Our proposition above works in the general situation of a space $X$ which is not necessarily finite, for which we needed the technical condition of the consistency of the probability measures. We now return to the finite setting, that is, below $U$ as before refers to a finite set of points in some bigger space $X$, where this consistency is not an issue.
We denote by $\mathcal{P}^n(U)$ the set of all probability measures over simplicial sets truncated at $n$ on the base set $U$. 
Also, recall the definition of $F^n(U)$ in \autoref{def:simpsetsoverX} as the collection of all truncated fuzzy simplicial sets over $U$.
We then have the following:
\begin{proposition}\label{prop:marginal}
The marginal map \begin{equation}
m: \mathcal{P}^n(U) \to \mathcal{F}^n(U), ~ p \mapsto   \mu_{p}, ~ \mu_p(\sigma) = p(S \geq \mathbf{S}(\sigma)) = p(S(\sigma) = 1)  \end{equation} is surjective.\end{proposition}
\begin{proof} Classical simplicial sets are the extremal points of the compact convex set of fuzzy simplicial sets. The result is a then standard from the perspective of convexity theory, see \cref{app:proofs} for details. \end{proof}

\subsection{Fuzzy Simplicial Sets from Distributions over Simplices}
Above we used probability measures over simplicial sets, which is quite a big space. We note that we may use the same construction for probability measures over simplices instead.
To mimic the construction above, we also want to obtain fuzzy weights by using an underlying poset-structure which will automatically give us monotonicity of the weights. 
To do this, we need to take care of the degenerate simplices, which we will do here by simply factoring them out:
Let thus $Z = \cup_{k=1}^n X^k$  be the set of all simplices up to order $n$, and consider the equivalence relation: 
\begin{equation}
\sigma \sim \sigma' \iff  \sigma \leq_{degeneracy} \sigma' \text{ or } \sigma' \leq_{degeneracy} \sigma
\end{equation}
where we recall that $\sigma \leq_{degeneracy} \sigma'$ means that there exists a sequence of degeneracy maps to obtain $\sigma$ from $\sigma'$.
Each such equivalence class  $[\sigma]$ has exactly one non-degenerate simplex, which we will denote $b([\sigma])$.
We can then define the face-order on the set of equivalence classes by 
\begin{equation}
[\sigma] \leq [\sigma'] \iff b([\sigma]) \leq_{face} b(\sigma'),
\end{equation}
 where we recall that $\leq_{face}$ implies a sequence of face maps to obtain one simplex from the other. 
With these constructions, we now have a partial order $\leq$ on $Z / \sim$, from which me may construct fuzzy weights, again by inducing them from the probability measures $P(Z / \sim)$ over equivalence classes. Note that in contrast to the previous setting, we would now sample (non-degenerate) simplices from such a measure instead of whole simplicial sets.
\begin{proposition}\label{prop:injection}
There is an injection \begin{equation} \mathcal{P}(Z / \sim ) \to  \mathcal{F}^n(X) \end{equation}
\end{proposition}

\begin{proof} \autoref{cor:cdminjective} in \cref{app:proofs} \end{proof}
\begin{proposition}
The injection from the previous proposition is not a surjection. 
\end{proposition}
\begin{proof}
Consider any standard simplicial set, identified with it's weight $\mu$, where there are two nondegenerate simplices $\sigma, \sigma'$ such that 
$$\mu(\sigma) = \mu(\sigma') = 1
$$
and where $\sigma, \sigma'$ are face-incomparable, i.e. neither one may be obtained by face maps from the other.
Let it further hold
$$ \mu(\sigma'') = 0$$
 for any $\sigma''$ such that both $\sigma \leq \sigma''$ and $\sigma' \leq \sigma''$. Then $\mu$ cannot be achieved by the map we have defined. Indeed, if $p$ is any probability measure such that 
\begin{equation}
p( \geq [\sigma]) = p(\geq [\sigma']) = 1,
\end{equation}
then necessarily there has to be a $\sigma''$ of the above form with 
\begin{equation}
p(\geq [\sigma'']) = 1
\end{equation}
and hence $\mu(\sigma'') = 1$ which violates the assumption.
\end{proof}

\begin{remark}
Again these constructions may be carried over to arbitrary posets, where there are standard results linking them.
\end{remark}

\subsection{Examples via Filtrations of Simplicial Complexes}
Before proceeding, we construct some examples.
\begin{example}
Let $(X,d)$ be a metric space and fix $r \geq 0$. For $U = \{x_{i_0},...,x_{i_n}\} \subset X$, consider the Vietoris-Rips complex $VR(U,r)$.
Then, we can construct a delta-measure on $\mathcal{S}^n(U)$, i.e.
\begin{equation}
p_{U}(S|r) = \delta(S = VR(U,r)).
\end{equation}
Then $m(p)$ simply corresponds to the weight in the VR complex, as 
\begin{equation}
p_{U}(S \geq \mathbf{S}(\sigma)|r) = \delta (VR(U,r) \geq \mathbf{S}(\sigma)) = \delta(\mu(VR(U,r))(\sigma) = 1).
\end{equation}
Thus, the obtained fuzzy simplicial set is simply given by the Vietoris-Rips complex $VR(X,r)$ - which thus also is a classical simplicial set.
\end{example}

The above example is tautological, but it will help us construct the next example. The Vietoris-Rips complex at a fixed scale $r$ does not capture all topological or metric information about the underlying vertices. Thus, usually, one wants to consider the whole filtration $VR(X,r), r \in \mathbb{R}_{\geq 0}$. This however yields a whole family of complexes instead of a single weight for each simplex. To obtain a single quantity for each simplex, we then may want to put a distribution on the scales, and average over this distribution. This is what we will do in the next example.

\begin{example}\label{ex:VR}
Assume we are randomly sampling the scales $r$ of the $VR$ complex according to some distribution $p(r)$, with cumulative distribution function $\phi(t) = \int_0^t p(r) dr$.
Let furthermore as before, given a fixed scale $r$,
\begin{equation}
p_U(S \vert r) = \delta(S = VR(U,r)).\end{equation}
Then we may average over the distribution of the scales to obtain 
\begin{equation}
p_U(S) = \int_0^{\infty} p_U(S \vert r) p(r) dr.
\end{equation}
\end{example}

\begin{proposition}
For a simplicial set $S \in \mathcal{S}^n(U)$ let 
\begin{equation} \begin{split}
d_m(S) &= \inf_{\sigma: S(\sigma) = 0} \max_{x_i,x_j \in \sigma} d(x_i,x_j) \\
d_M(S)&= \sup_{\sigma: S(\sigma) = 1} \max_{x_i,x_j \in \sigma} d(x_i,x_j),
\end{split} \end{equation}
with the convention $\inf(\emptyset) = \infty, \sup(\emptyset) = -\infty$.
Then the probability under the above distribution  $p_U$ is given by 
\begin{equation} \begin{split} 
p_U(S) &= \left(\phi\left(d_m(S)\right) - \phi\left(d_M(S)\right)\right) \delta(\left[d_m(S) > d_M(S)\right])\\
&= \left(\phi\left(d_m(S)\right) - \phi\left(d_M(S)\right)\right) \delta([\exists r: S  = VR(U,r)]).
\end{split} \end{equation}
\end{proposition}
\begin{proof}
This follows since we can write 
\begin{equation} \begin{split}
\delta(VR(U,r) = S) &= \prod_{\sigma:  S(\sigma) = 0} \delta([ \max_{x_i,x_j \in \sigma }  d(x_i,x_j) > r])  \prod_{\sigma : S(\sigma) = 1} \delta([\max_{x_i,x_j \in \sigma }d(x_i,x_j) \leq r]) \\&= \delta(  d_m(S)  >r) \delta( d_M(S) \leq r )
\end{split} \end{equation}
and hence 
\begin{equation} \begin{split}
\int_0^\infty p(r) \delta(VR(U,r) = S)  dr &= \delta([d_m(S) > d_M(S)])\int_{d_M(S) }^{d_m(S)} p(r) dr \\
&= \left(\phi \left(d_m(S)\right) - \phi \left(d_M(S)\right)\right) \delta(\left[d_m(S) > d_M(S)\right])
\end{split} \end{equation}
as claimed. 
\end{proof}
Now, $d_m(S) > d_M(S)$ is a necessary and sufficient condition for $S$ to be a VR complex, namely, it then is the $VR$ complex at scale $d_M(S)$. Thus, the probability is determined by the value of $\phi$ at just two particular distances - this is intuitively clear: the order of distances has to be respected, since an edge corresponding to a smaller distance will always appear in a VR complex before one of a larger distance.  This restricts the possible value of scale $r$ for a given $S$ in between the maximum distance that still is in the complex and the minimum which isn't, which is exactly the formula one gets.
\begin{figure}
\centering
\includegraphics[width=\textwidth]{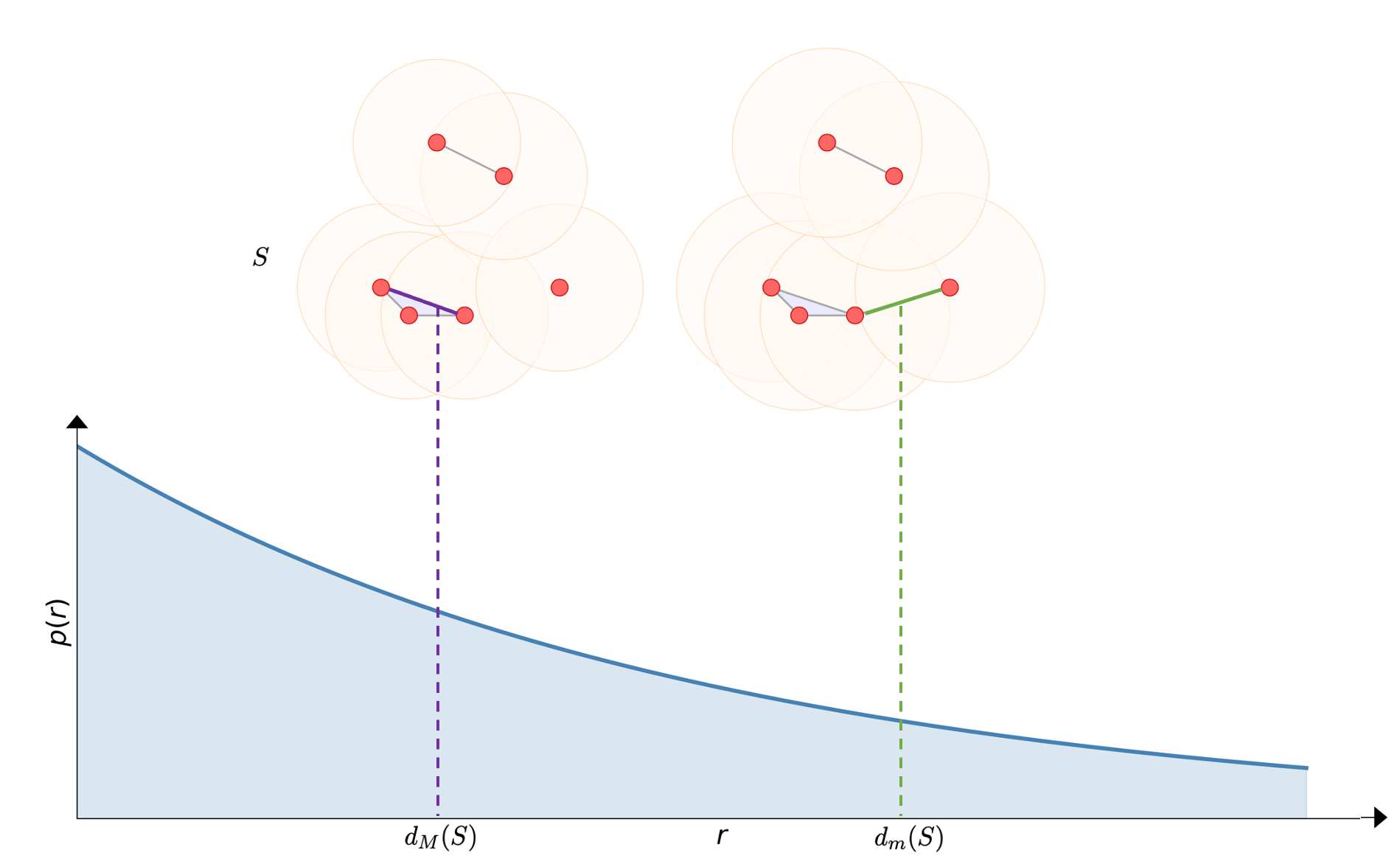}
\caption{Illustration of the probability distribution defined above. We sample radii according to a probability distribution $p(r)$. To determine the probability of a certain simplicial set $S$ like the one on the left, we first have to check whether it is a valid element of the VR filtration, else its probability is zero. Then, the probability is determined by integrating $p(r)$ from $d_M(S)$ to $d_m(S)$. $d_M(S)$ is the radius where the last simplex was added to $S$ (purple edge, alternatively the filled triangle which appears at the same time), and $d_m(S)$ is the lowest radius where a new simplex would be added (green edge).}
  \label{Fig:5}
\end{figure}

We will provide another perspective on this below, but let us first study the marginal distributions:
\begin{corollary}
The 'marginal distributions' under $p_U(S\vert r)$ are given by 
\begin{equation}
p(S([x_{i_0},...,x_{i_k}]) = 1) =p( \max_{j,l} d(x_{i_j},x_{i_l}) \leq r) = 1-\phi( \max_{j,l} d(x_{i_j},x_{i_l})).
\end{equation}
\end{corollary}
\begin{proof}
Note that we have
\begin{equation} \begin{split}
p(S([x_{i_0},...,x_{i_k}]) = 1) &= \sum_{S : S([x_{i_0},...,x_{i_k}]) = 1} p(S) \\ &= \sum_{S : S([x_{i_0},...,x_{i_k}]) = 1,\exists r: S  = VR(U,r)} \left(\phi \left(d_m(S)\right) - \phi \left(d_M(S)\right)\right) 
\end{split} \end{equation}
This sum now runs over all possible $VR$ complexes that have a value of $1$ at the simplex in question. 
 This is a telescope sum, thus what remains is the value at the largest complex minus the value at the smallest complex for which this is still true. The largest complex is the one where all simplices are present, hence $d_m(S) = \infty$. The smallest one has to be at scale $\max_{j,l} d(x_{i_j},x_{i_l})$.
This then gives the formula
\begin{equation}
p(S([x_{i_0},...,x_{i_k}]) = 1) = \phi(\infty) - \phi(\max_{j,l} d(x_{i_j},x_{i_l})) = 1 -  \phi(\max_{j,l} d(x_{i_j},x_{i_l}))
\end{equation}
as claimed.
\end{proof}
As we have stated before, these marginal distributions induce a fuzzy simplicial set. We note that the resulting fuzzy weights in this case are directly obtained by applying a function simplex-wise to the diameters of the simplices. 
This is a rather simple construction, however, it conveys the fuzzy weights with probabilistic meaning. 
\begin{example}\label{ex:exponential}
For $p(r) = \frac{1}{\nu} \exp(- \frac{r}{\nu})$ an exponential distribution with parameter $\frac{1}{\nu}$, one has that $\phi(t) = 1- \exp\left(-\frac{t}{\nu}\right)$,which results in marginals of the form
\begin{equation} \begin{split}
\mu([x_i,x_j]) &= p(S([x_i,x_j]) = 1\vert X)  = \exp\left(-\frac{d([x_i,x_j])}{\nu}\right) \\
\mu([x_{i_0},...,x_{i_j}]) & = \exp \left(-  \max_{k,l} \frac{d([x_k,x_l])}{\nu}\right)
\end{split} \end{equation}
\end{example}

\begin{remark}
The VR complex has the property of being completely determined by its 1-simplices or edges. In terms of weight functions, this means that the inequality for the face-inclusion becomes an equality:
\begin{equation}
\mu(\sigma) = \min_{\sigma' \leq \sigma} \mu(\sigma').
\end{equation}
This property carries directly over to the fuzzy weights in the above procedure, that is, the fuzzy simplicial set is completely determined by the weight of its 1-simplices.
\end{remark}

\begin{example}\label{ex:cech}
Above we have used VR complexes, which are often used in practice for their simplicity. From a theoretical standpoint, due to the nerve theorem, the \v{C}ech filtration is actually more relevant. 
We get a very similar distribution when we start from this filtration:
Integrating  $\delta$-distributions over $C(U,r)$ with $p(r)$ result in a distribution 
\begin{equation} \begin{split}
p(S \vert X) &= \left(\phi\left(d_m(S)\right) - \phi\left(d_M(S)\right)\right) \delta\left(d_m(S) > d_M(S)\right) \\
&\left(\phi\left(d_m(S)\right) - \phi\left(d_M(S)\right)\right) \delta(\exists r: S  = C(U,r))
\end{split} \end{equation}
with now 
\begin{equation} \begin{split}
d_m(S) &=  \inf_{\sigma: S(\sigma) = 0}\inf_{y} \max_{k: x_{i_k} \in \sigma} d(y,x_{i_k})\\
d_M(S)&= \sup_{\sigma: S(\sigma) = 1} \inf_{y} \max_{k: x_{i_k} \in \sigma} d(y,x_{i_k}).
\end{split} \end{equation}
The proof is of course the same as for the VR-filtration.
\end{example}

\subsection{Comparison and merging operations}
We now want to study some operations we can perform on fuzzy simplicial sets. 
First, a standard way to compare fuzzy sets is the following. 
\begin{definition}
Given two fuzzy sets $(A, \mu_1), (A,\mu_2)$ with the same underlying set $A$, we define the fuzzy cross entropy between them as
\begin{equation}
\operatorname{CE}((A, \mu_1) \vert \vert (A,\mu_2)) = \sum_{a} \mu_1(a) \ln \frac{\mu_1(a)}{\mu_2(a)} + (1-\mu_1(a))  \ln \frac{1-\mu_1(a)}{1-\mu_2(a)}
\end{equation}
Given two fuzzy simplicial sets $(S,\mu_1), (S,\mu_2)$ one may define a fuzzy cross entropy as 
\begin{equation}
\operatorname{CE}((S, \mu_1) \vert \vert (S,\mu_2)) = \sum_{n} w(n) CE((S_n,\mu_1)\vert \vert(S_n,\mu_2))
\end{equation}
with an additional weighting factor $w(n)$.
\end{definition}
Now for two probability distributions over a simplicial set, we have a standard tool from probability theory to compare their distributions through the Kullback-Leibler divergence.
\begin{definition}
For two probability distributions $p,q$ over $\mathcal{S}^n(X)$, the Kullback-Leibler divergence is 
\begin{equation}
\operatorname{D_{KL}}(p \vert \vert q) = \sum_{S \in \mathcal{S}^n(X)} p(S) \ln \frac{p(S)}{q(S)}
\end{equation}
\end{definition}
We now want to investigate how this divergence compares to the fuzzy cross entropy. To do so, we will assume our distributions have more structure. 
For this, we recall some standard definitions (consult e.g. \cite{koller2009probabilistic} for reference).
\begin{definition}
Let $G = (V,E)$ be a directed, acyclic graph (DAG). A collection of random variables $Z = (Z_{v})_{v\in V}$, indexed by vertices of the graph, is a Bayesian network with respect to $G$ if 
\begin{equation}
p(Z) = \prod_{v} p(Z_v \vert Z_{\pi(v)}) 
\end{equation}
where $\pi(v)$ are the parents of $v$ in $G$ - those are simply the nodes having a directed edge to $v$.
\end{definition}
In our situation, we have a specific DAG structure that comes from the partial order:
\begin{definition} 
Let $p$ be a probability distribution over simplicial sets in $\mathcal{S}^n(U)$. We may treat the indicator values $S(\sigma)$ as binary random variables indexed by simplices $\sigma$, $S(\sigma)$ signaling that the simplex $\sigma$ is present in the random simplicial set $S$.
We will then call $p$ ``locally Markov'' if the collection of random variables $S(\sigma)$ is a Bayesian network with respect to the graph induced by the partial order of simplices which we have defined for \autoref{prop:injection}.
\end{definition}
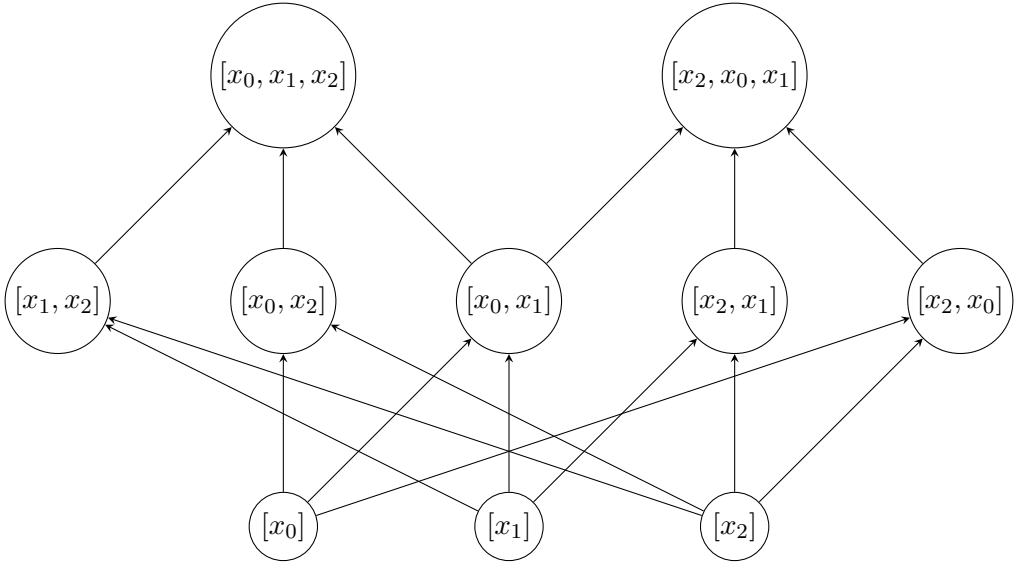
\begin{figure}
\centering
\begin{tikzpicture}[
    node distance=1.5cm,
    every node/.style={fill=white, draw, circle, inner sep=2pt},
    level 1/.style={sibling distance=3.5cm},
    level 2/.style={sibling distance=2cm},
    ->,
    >=stealth
]

\node (x0x1x2) at (-3,3) {$[x_0, x_1, x_2]$};

\node (x1x2) at (-6,0) {$[x_1, x_2]$};
\node (x0x2) at (-3,0) {$[x_0, x_2]$};
\node (x0x1) at (0,0) {$[x_0, x_1]$};

\node (x0) at (-3,-3) {$[x_0]$};
\node (x1) at (0,-3) {$[x_1]$};
\node (x2) at (3,-3) {$[x_2]$};

\draw (x1x2)--(x0x1x2);
\draw (x0x2)--(x0x1x2);
\draw (x0x1)--(x0x1x2);

\draw (x1)--(x1x2);
\draw (x2)--(x1x2);

\draw (x0)--(x0x2);
\draw (x2)--(x0x2);

\draw (x0)--(x0x1);
\draw (x1)--(x0x1);

\node (x2x0x1) at (3,3) {$[x_2, x_0, x_1]$};

\node (x2x1) at (3,0) {$[x_2, x_1]$};
\node (x2x0) at (6,0) {$[x_2, x_0]$};


\draw (x0x1)--(x2x0x1);
\draw (x2x1)--(x2x0x1);
\draw (x2x0)--(x2x0x1);

\draw (x2)--(x2x1);
\draw (x1)--(x2x1);

\draw (x2)--(x2x0);
\draw (x0)--(x2x0);

\end{tikzpicture}
\caption{The DAG structure we are assuming, here for two triangles, simply is the poset-structure of face maps. Degenerate simplices are omitted as usual.}
\end{figure}
The structure of a DAG imposes a special form on the KL-divergence as we recall in the next lemma. 
\begin{lemma}\label{lem:factorizeKL}
Let $p,q$ be two probability distributions over a collection $Z_v, v \in V$ for some DAG $G$, with $V$ and the range of the $Z_v$ assumed finite for simplicity.
Assume $q$ constitutes a Bayesian network with respect to the graph $G$. 
Then, the Kullback-Leibler divergence factorizes as 
\begin{equation}
\operatorname{D_{KL}}(p \vert \vert q) = -H(p)  - \sum_{v \in V} \sum_{z_v,z_{\pi(v)}}  p( Z_v = z_v, Z_{\pi(v)}=z_{\pi(v)}) \ln q(Z_v = z_v \vert Z_{\pi(v)} =z_{\pi(v)}),
\end{equation}
where $H(p)$ is the entropy of $p$ that does not depend on $q$.
\end{lemma}
\begin{proof}
First, write 
\begin{equation}
\operatorname{D_{KL}}(p \vert \vert q) = \sum_{z} p(z) \ln p(z) - \sum_{z} p(z) \ln q(z) = -H(p) - \sum_{z} p(z) \ln q(z)
\end{equation}
where $z$ in the sum ranges over all possible assignments $(z_{v_1},...,z_{v_n})$ of values to the respective random variables, 
that is e.g. $p(z) = p(Z_{v_1} = z_{v_1},..., Z_{v_n} = z_{v_n})$.
Then by $q$ being locally Markov,
 $\ln q(z) = \ln \prod_{v} q(Z_v = z_v, \vert Z_{\pi(v)}=z_{\pi(v)})  = \sum_v \ln q(Z_v = z_v, \vert Z_{\pi(v)}=z_{\pi(v)})$. Then in each summand, all other variables are marginalized out of $p(z)$. \end{proof}
Note that similarly, if both $p,q$ are locally Markov, also the entropy term in the KL-divergence decomposes in the same fashion.
Before we study the KL-divergence, we need one more simple lemma,
\begin{lemma} \label{lem:marginalprob}
Let $p$ be a probability distribution over simplicial sets. Then necessarily
\begin{equation}
p\left(S(\sigma) = 1, S(d_j(\sigma)) = 1,  \forall j \right)   =  p(S(\sigma) = 1).
\end{equation}
\end{lemma}
\begin{proof}
We note that whenever any of its faces is not present, $\sigma$ may also not be present, thus has probability zero.
Hence,
\begin{equation}
p(S(\sigma) = 1) = \sum_{s^j}  p\left(S(\sigma) = 1, S(d_j(\sigma)) = s^j\right) = p\left(S(\Sigma) = 1, S(d_j(\sigma)) = 1 \forall j \right)
\end{equation}
\end{proof}
Now we may give a more explicit form of the KL-divergences.
\begin{proposition}
Let $p,q$ be probability distributions over simplicial sets which are locally Markov. Then, their Kullback-Leibler divergence is given by 
\begin{equation}
  \begin{split}
\operatorname{D_{KL}}(p \vert \vert q) =  \sum_{\sigma}  &p(S(\sigma) = 1)  \big\{p\left(S(\sigma) = 1\vert  S(\pi(\sigma)) = 1\right)\ln \frac{p( S(\sigma) = 1 \vert S(\pi(\sigma)) = 1)}{q( S(\sigma) = 1 \vert S(\pi(\sigma)) = 1)} \\+ &\left[1-p\left(S(\sigma) = 1\vert  S(\pi(\sigma)) = 1\right) \right] \ln \frac{(1-p(S(\sigma) = 1 \vert S(\pi(\sigma)) = 1))}{(1-q(S(\sigma) = 1 \vert S(\pi(\sigma)) = 1))}\big\}
  \end{split}
\end{equation} \end{proposition}
\begin{proof}
By $p,q$ being locally Markov, as in \autoref{lem:factorizeKL} (and where now the entropy decomposes equivalently as $p$ is also locally Markov), the KL divergence may be written as \begin{equation} \begin{split}
\operatorname{D_{KL}}(p \vert \vert q) &= \sum_S p(S) \ln\frac{p(S)}{q(S)}  \\ &
=\sum_{\sigma} \sum_{s_\sigma, s_{\pi(\sigma)}} p(S(\sigma) = s_\sigma, S(\pi(\sigma)) = s_{\pi(\sigma)}) \ln \frac{p( S(\sigma) = s_\sigma \vert S(\pi(\sigma)) = s_{\pi(\sigma)})}{q( S(\sigma) = s_\sigma \vert S(\pi(\sigma)) = s_{\pi(\sigma)})}.
\end{split} \end{equation}
Now, the only terms that remain in this sum are those where all $s_{\pi(\sigma)} =1$. This is because: From \autoref{lem:marginalprob}, if $s_\sigma =1$ but any of the $s_{\pi(\sigma)} = 0$, then the joint probability in $p$ becomes zero. On the other hand, if $s_{\sigma} = 0$ and any of the $s_{\pi(\sigma)} = 0$, the conditional probability becomes $1$ and hence the logarithm terms (in both $p$ and $q$) evaluate to zero. 
Thus what remains is 
\begin{equation} \begin{split}
\sum_{\sigma}  &p(S(\sigma) = 1, S(\pi(\sigma)) = 1) \ln \frac{ p( S(\sigma) = 1 \vert S(\pi(\sigma)) = 1)}{q( S(\sigma) = 1 \vert S(\pi(\sigma)) = 1)} \\+ &p(S(\sigma) = 0, S(\pi(\sigma)) = 1) \ln  \frac{ p( S(\sigma) = 0 \vert S(\pi(\sigma)) = 1)
}{q( S(\sigma) = 0 \vert S(\pi(\sigma)) = 1)}
\end{split} \end{equation}  
which by simple manipulation yields the result.
\end{proof}
This shows that in general if we have two probability distributions $p,q$ over simplicial sets, both locally Markovian, then the fuzzy cross entropy of their associated fuzzy simplicial sets is in general not equal to their Kullback-Leibler divergence, that is.
\begin{equation}
\operatorname{CE}(\mu_p \vert \vert \mu_q) \neq \operatorname{D_{KL}}(p \vert \vert q).
\end{equation}
However, in the special case where the distribution is truncated at $1-$simplices, and vertices are there with probability $1$, then we do have equality, as this implies
\begin{equation}
p( S(\sigma) = 1 \vert S(\pi(\sigma)) = 1)  = p(S(\sigma) = 1), q( S(\sigma) = 1 \vert S(\pi(\sigma)) = 1) = q(S(\sigma) = 1).
\end{equation}
\begin{corollary}\label{cor:kl}
If $p,q$ are locally Markov distributions on $\mathcal{S}^1(X)$, where all $0$-simplices have probability $1$, then 
\begin{equation}
\operatorname{CE}(\mu_p \vert \vert \mu_q) = \operatorname{D_{KL}}(p \vert \vert q).
\end{equation}
\end{corollary}
Note that the above special case essentially means we have a distribution where the presence of each edge is independent from all others.
\begin{remark}
This special case is the one that is used in the standard implementation of the UMAP algorithm. Thus, from the probabilistic perspective of this framework, UMAP discards all interdependence of simplices by only comparing marginal distributions.
\end{remark}

\subsubsection{Merging}
As stated in the preliminaries, for fuzzy sets there are natural merging operations, which are the t-(co)-norms. 
It is easy to see that, by the monotonicity criterium, merging two fuzzy simplicial sets via a t-(co)-norm yields again a fuzzy simplicial set.
\begin{lemma}
For $(S,\mu_1)$, $(S,\mu_2)$ two fuzzy simplicial sets on the same base sets, $(S,T(\mu_1,\mu_2))$ is again a fuzzy simplicial set, where $T$ is a t-(co)-norm.
\end{lemma}
We now want to understand how t-conorm operations on the generated fuzzy simplicial sets could arise from the underlying probability distributions over simplicial sets.
Since simplicial sets take weights in $\{0,1\}$, we can naturally take Boolean operations on them. Given two probability measures $p_1,p_2$ on $\mathcal{S}^n(U)$, there is a natural way to induce a new merged measure from them via a Boolean operation.
\begin{definition}
For an operation $*$ on $\{0,1\}^2$ we define the merged probability measure $p_1*p_2$ as 
\begin{equation}
(p_1 * p_2) (\{ S \})= p[S_1 * S_2 = S], 
\end{equation}
where on the right  $S_1,S_2$ are sampled independently from $p_1,p_2$, that is on the right $p$ is the measure on the product space with $p[(S_1,S_2)] = p_1(S_1) p_2(S_2)$. 
On the right hand side, the operation $*$ is applied elementwise to the weight functions of $S_1,S_2$.

\end{definition}

Now we study what merging such probability measures does to the induced fuzzy objects.
Recall that $\mathbf{S}(\sigma)$ is the minimal simplicial set containing a simplex $\sigma$. 
Also recall that these minimal elements are comparable (with respect to our usual order on simplical sets) with all other simplicial sets.

\begin{proposition}
Taking the maximum/OR operation induces the probabilistic t-conorm on the underlying fuzzy simplicial sets, that is 
\begin{equation}
\mu_{p_1 \lor p_2} = \mu_{p_1} + \mu_{p_2} -\mu_{p_1}\mu_{p_2}.
\end{equation}
\end{proposition}
\begin{proof}
Note that elementwise
\begin{equation}
S_1 \lor S_2  = \max(S_1,S_2)
\end{equation}
Hence, we calculate
\begin{equation}
  \begin{split}
\mu_{p_1 \lor p_2}(\sigma) &= p[\max (S_1,S_2) \geq \mathbf{S}(\sigma)] \\
&= 1- p[\max(S_1,S_2)< \mathbf{S}(\sigma)] \\
&= 1-p[S_1< \mathbf{S}(\sigma) \land S_2 < \mathbf{S}(\sigma)] \\
&= 1- p[S_1 < \mathbf{S}(\sigma)] p_2[S_2 < \mathbf{S}(\sigma)] \\
&= 1- (1-p_1[S_1 \geq \mathbf{S}(\sigma)]) (1-p_2[S_2 \geq \mathbf{S}(\sigma)]) \\
&= p_1[S_1 \geq \mathbf{S}(\sigma)] +  p_2[S_2 \geq \mathbf{S}(\sigma)]-p_1[S_1 \geq \mathbf{S}(\sigma)] p_2[S_2 \geq \mathbf{S}(\sigma)]\\
&= \mu_{p_1}(\sigma) + \mu_{p_2}(\sigma) -\mu_{p_1}\mu_{p_2}(\sigma)
\end{split}
\end{equation}
\end{proof} 
\begin{proposition}
Taking the minimum/AND operation induces the product t-norm on the underlying fuzzy objects, that is 
\begin{equation}
\mu_{p_1\land p_2} = \mu_{p_1}\mu_{p_2} 
\end{equation}
\begin{proof}
As for the dual case.
\end{proof}
\end{proposition}

\begin{example}
Recall our construction of delta-distributions of Vietoris-Rips complexes
\begin{equation}
p(S \vert r) = \delta_{VR^d(r)}(S),
\end{equation}
which is then averaged with $p(r)$.
When we now want to combine two such distributions by means of a logical operation, the order of operations matter.
For example, we have for two metrics $d_1,d_2$
\begin{equation}
\delta_{VR^{d_1}(r)}(S) \land \delta_{VR^{d_2}(r)}(S) = \delta_{VR^{\max(d_1,d_2)}}(S).
\end{equation}
Now taking the average of this with $p(r)$, and then taking the fuzzy weights $\mu$, results in weights (for simplicity, we consider weights on the edges)
\begin{equation}
\mu([x_i,x_j]) = 1-\phi\left( \max(d_1,d_2)(x_i,x_j)\right).
\end{equation}
On the other hand, if we first take the average for each individual metric and then take the intersection as in the previous proposition, we obtain
\begin{equation}
\mu([x_i,x_j])  =   (1-\phi\left(d_1(x_i,x_j)\right))(1-\phi\left(d_2(x_i,x_j)\right))
\end{equation}
which is clearly not equal to taking the maximum of the distances. 
Thus
\begin{equation}
\int \delta_{VR^{d_1}(r)} \land \delta_{VR^{d_2}(r)} p(r)dr  \neq  \left( \int \delta_{VR^{d_1}(r)} p(r) dr \right) \land \left(\int \delta_{VR^{d_2}(r)} p(r) dr \right)
\end{equation}
\end{example}

\section{Dimensionality Reduction via Probabilistic Fuzzy Simplicial Sets}
The probabilistic representation has immediate consequences for existing methods in dimensionality reduction. In particular, it provides a principled interpretation of UMAP and suggests several variants. We will first reinterpret UMAP in this light and then present some possible alternatives.
\subsection{UMAP}

UMAP is a celebrated algorithm which is widely used for data visualization. The theoretical motivation behind UMAP hinges on fuzzy simplicial sets, as we will now quickly explain. 
Indeed, given a finite dataset $X = x_1,...,x_n$ in some metric space, UMAP first constructs local extended pseudo-metric spaces $(X,d_i)$, where the local distances $d_i$ are of the form 
\begin{equation}
d_i (x,y) = \begin{cases} \frac{d(x,y)- \rho_i}{\sigma_i},&   \text{if } *(x,y), \\
0,& x = y, \\
 \infty,& \text{else.} \end{cases}
\end{equation}
Here, $*(x,y)$ denotes the condition that either $x = x_i$ and $y$ is among the $k$ nearest neighbors of $x_i$, or vice versa. The normalization factor $\sigma_i$ is supposed to account for effects of data density, and the subtraction of $\rho_i$ alleviates the curse of dimensionality. If the metric space our data lives in is Euclidean space, one may understand these local metrics in terms of Riemannian geometry as local neighborhoods of an unknown manifold where the data is distributed on, and where the neighborhoods are such that the distances on the manifold are well approximated by the Euclidean ones.
Having localized the metrics comes with the need to merge them again to obtain a global metric on the data. 
UMAP achieves this by transferring the local metric spaces to fuzzy simplicial sets $(S^i,\mu^i) \in F^1(X)$, where the weights are given by $\mu^i = \exp(-d_i)$. 
Weights on higher-order simplices than edges are not introduced in the method due to computational constraints, but the framework we develop here is intentionally general to be able to accomodate higher order merging in the same framework. 

Once transferred to fuzzy simplicial sets, these local spaces are merged via a t-conorm. 
Then, to obtain low-dimensional representations of the data-points, UMAP uses a Force-Directed-Graph Layout, based on minimizing the fuzzy cross entropy 
\begin{equation}
CE(X,Y) = \sum_{(i,j)} \mu([x_i,x_j]) \ln \frac{\mu[x_i,x_j]}{\nu^Y([x_i,x_j])} + (1-\mu([x_i,x_j])) \ln \frac{1-\mu[x_i,x_j]}{1-\nu^Y([x_i,x_j])}.
\end{equation}
Here, $\nu^Y$ is a weight generated by distances of the low-dimensional points $Y = y_1,...,y_n$ which are optimized to minimize the cross entropy. In UMAP, it is 
$\nu(Y)[x_i,x_j] = \frac{1}{1 + a d(Y_i,Y_j)^{2b}}$, where $a,b$ are hyperparameters and $d$ is the distance is the low dimensional space.
This corresponds to a distribution function with heavier tails in the low-dimensional space, given by 
\begin{equation}
\phi^Y(r) =  \frac{a r^{2b}}{1+ a r^2b}.
\end{equation}
This may be identified as the cdf of a log-logistic distribution, which in standard form is written as $\frac{1}{1 + \frac{r}{\alpha}^{-\beta}}$, where $\beta = 2b$ and $\alpha = a^{-\frac{1}{2b}}.$
Thus, the following is clear:
\begin{proposition}
The local fuzzy weights in UMAP may be obtained from probability distributions over Vietoris-Rips filtrations, based on the local pseudo-metric $d_i$, and the metric $d$, respectively, where the distributions are $ p^X(r) = \exp(-r),p^Y(r) = \frac{2ab r^{2b-1}}{(1+ a r^2b)^2}$.
\end{proposition}
 \begin{proof}
This follows directly from $\autoref{ex:VR}$ and $\autoref{ex:exponential}$ and differentiating $\phi^Y(r)$.
\end{proof}
In particular, one may interpret the fuzzy weights in UMAP as the probabilities of observing a particular edge in a union of VR complexes, when the scales are sampled from an exponential distribution, independently at each datapoint.
The fuzzy cross entropy then results from the special case of independence assumed over all edges as explained in \autoref{cor:kl}.

From this perspective on the scales, one immediately obtains a generalization of UMAP:
\begin{corollary}
Any cumulative distribution functions $\phi^X,\phi^Y$ of a probability density over non-negative reals yields weights $\mu([x_i,x_j]) = 1-\phi(d([x_i,x_j]))$, which offer a generalization of UMAP. This boils down to using the force-directed graph layout based on the loss
\begin{equation}
CE(X,Y) = \sum_{(i,j)} (1-\phi^X(d_i(x_i, x_j)) \ln \frac{1-\phi^X(d_i(x_i, x_j))}{1- \phi^Y(d(y_i, y_j))} + (\phi^X(d_i(x_i, x_j)) \ln \frac{\phi^X(d_i(x_i, x_j))}{\phi^Y(d(y_i, y_j))}.
\end{equation}
\end{corollary}
Interestingly, one naturally arrives at the requirement of using cumulative distribution functions purely from considerations on how to transfer metrics to fuzzy weights. We elaborate on this in appendix \ref{app:OnAppropriateWeight-to-distanceFunctions}.
To illustrate, we apply the UMAP pipeline (\url{https://github.com/lmcinnes/umap}) to a toy example (MNIST, $N=10000$ data points). We use the standard settings of UMAP and only change the low-dimensional affinities to come from a W
Weibull$(\lambda, k)$ distribution with parameters $\lambda = 1$ and $k$ varying. For shrinking $k$, the distribution has heavier tails, which leads to clusters separating more in the embedding. This seems to correspond to the attraction-repulsion spectrum observed in \cite{bohm2020unifying}, albeit here parametrized via the shape of the distribution.

\begin{figure}
\includegraphics[width=0.3 \textwidth]{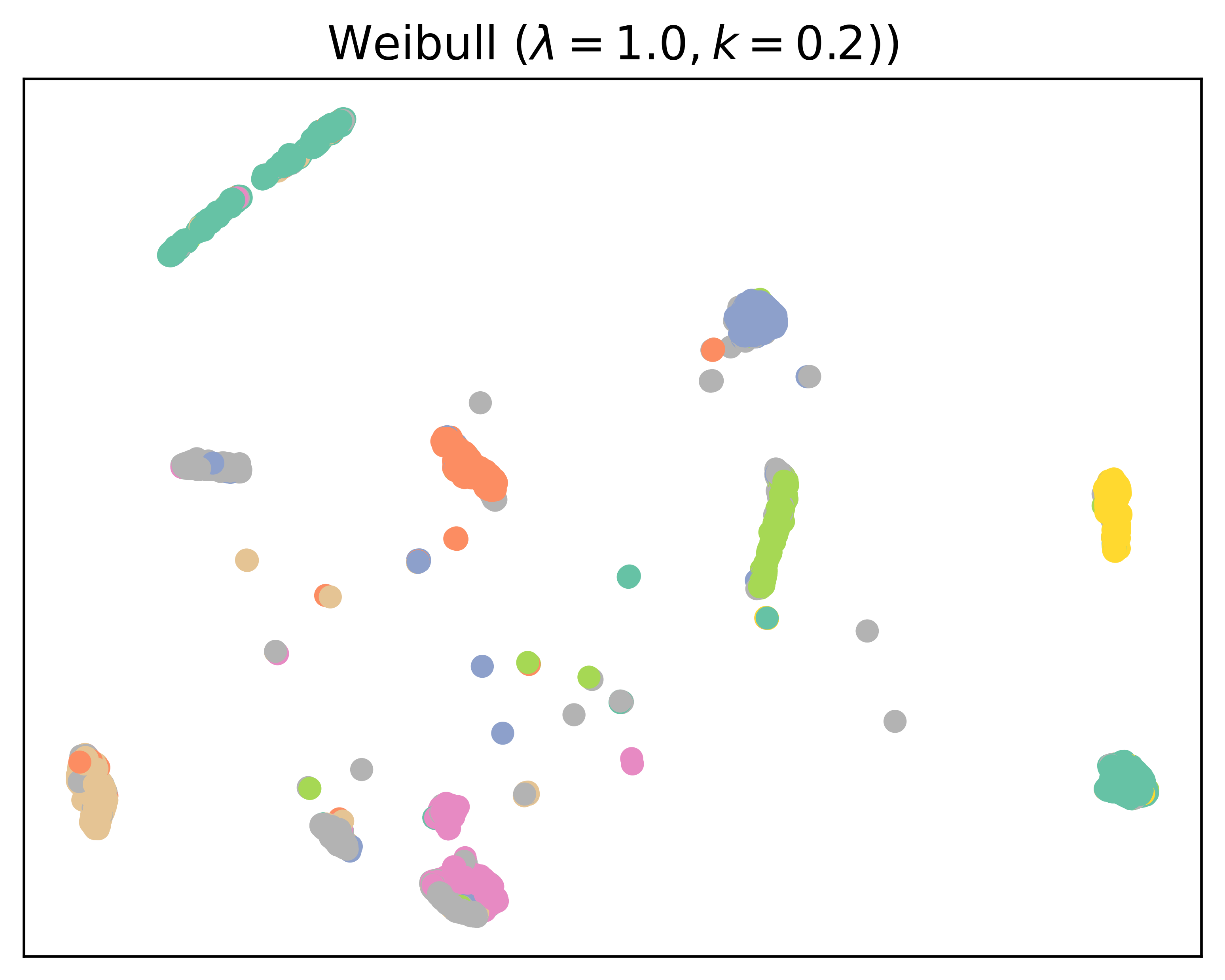}
\includegraphics[width=0.3 \textwidth]{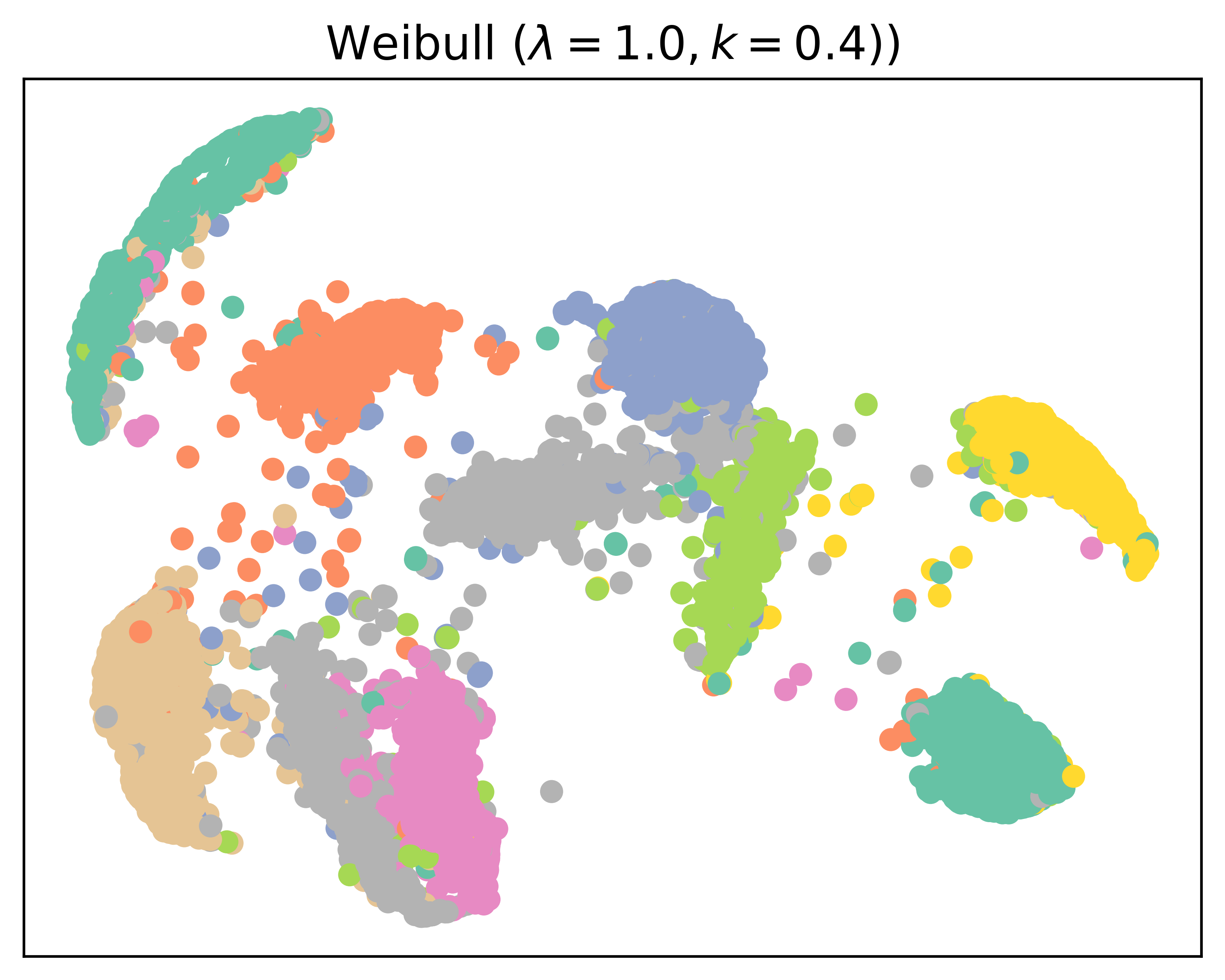}
\includegraphics[width=0.3 \textwidth]{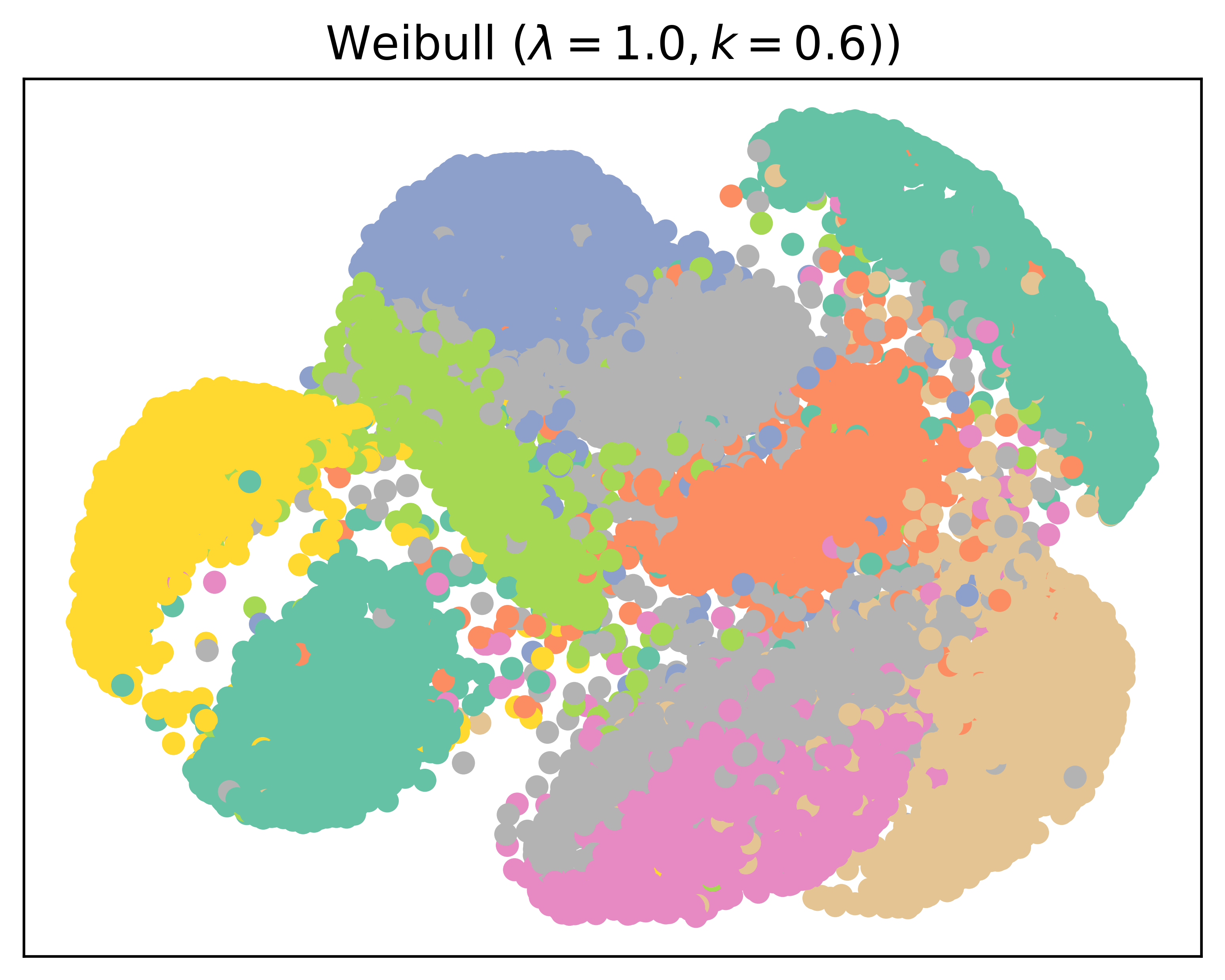}
\centering
\includegraphics[width=0.3 \textwidth]{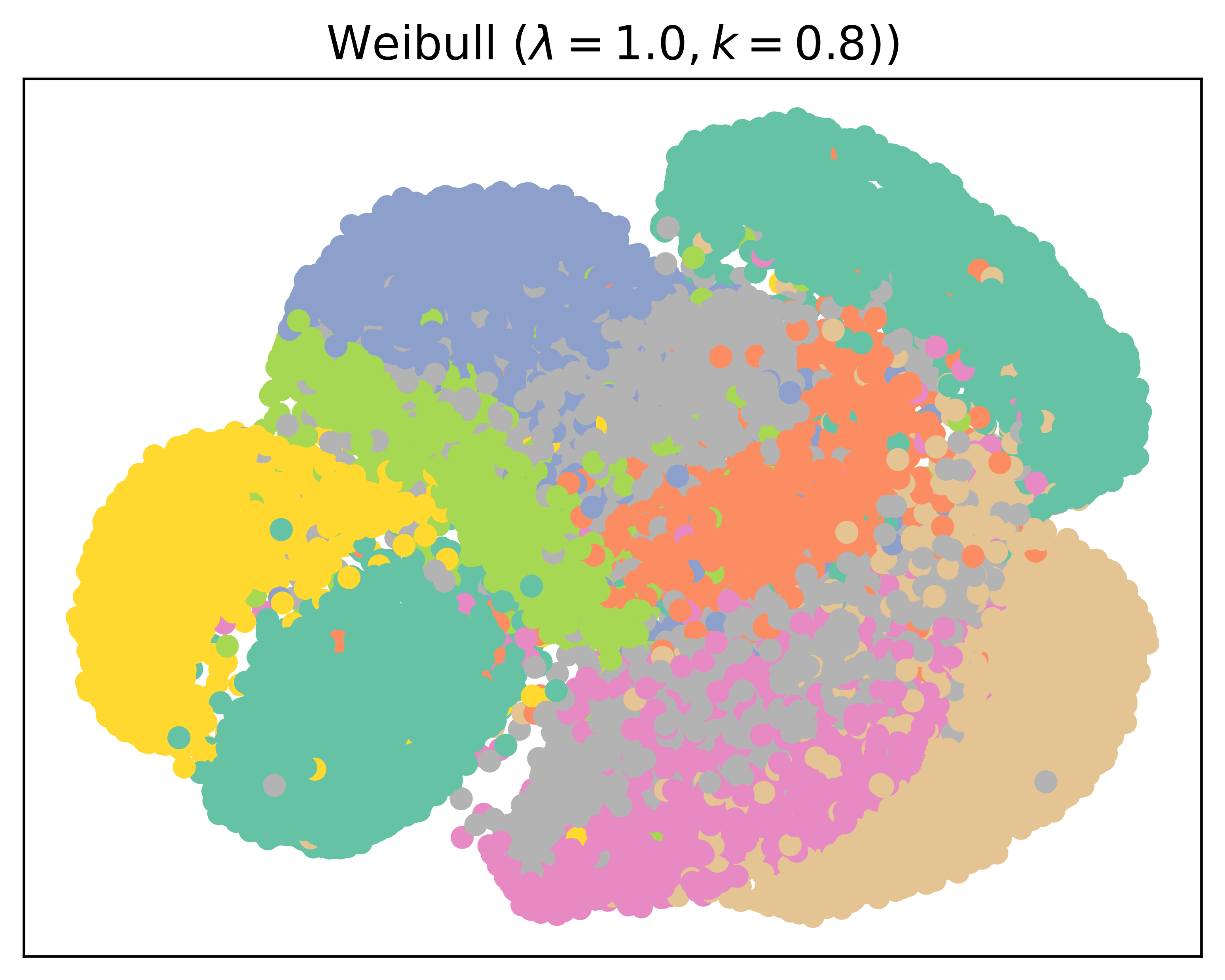}
\caption{UMAP embedding on MNIST (N=10000), using a Weibull distribution for $\phi^Y$ with varying parameters}.\end{figure}

We will now discuss further possible generalizations that this framework naturally suggests.
\subsection{\v{C}UMAP  - UMAP on \v{C}ech complexes}
\label{sec:umapCech}
Another natural generalization of UMAP that arises from this framework is to use the \v{C}ech complex to define fuzzy weights.
Following example \autoref{ex:cech}, for a dataset $X = x_1,...,x_n$ in a Euclidean space $\mathbb{R}^d$ we may define weights 
\begin{equation}
\mu([x_{i_1},....x_{i_k}]) = 1-\phi\left(\min_{y \in \mathbb{R}^d}  \max_{l} d(x_{i_l},y) \right) 
\end{equation}
where $\phi$ is again a cumulative distribution function, e.g. $\phi = 1-\exp(- \cdot)$. Note that for edges $[x_i,x_j]$ one simply has 
\begin{equation}
\min_{y \in \mathbb{R}^d}  \max (d(x_i,y), d(x_j,y) )  = \frac{1}{2} d(x_i,x_j) \end{equation}
that is, for edges this is simply equivalent to the VR-complex up to rescaling.
For triangles, one has the following formula:
\begin{equation}
\small
\min_{y \in \mathbb{R}^d}  \max (d(x_i,y), d(x_j,y),d(x_k,y) )  =  \begin{cases} \frac{d_{\max}}{2}, & d(x_i,x_j)^2 + d(x_j,x_k)^2 + d(x_k,x_i)^2 \leq 2 d_{\max}^2 \\ R(x_i,x_j,x_k) & \text{else}.\end{cases}
\end{equation}
where $d_{\max} = \max {d(x_i,x_j),d(x_j,x_k), d(x_k,x_i)}$ is the longest side length of the triangle and $R(x_i,x_j,x_k)$ is the circumradius of the smallest enclosing ball of the three points $x_i,x_j,x_k$, which may be for example calculated as 
\begin{equation} \begin{split}
R(x_i,x_j,x_k) &= \frac{d(x_i,x_j)d(x_j,x_k)d(x_k,x_i)}{4 \sqrt{s (s- d(x_i,x_j)) (s-d(x_j,x_k))(s-d(x_k,x_i))}}, \\ s &= \frac{d(x_i,x_j) + d(x_k,x_i) + d(x_k,x_i)}{2} \nonumber
\end{split} \end{equation}
For higher order simplices, generally a closed form formula will not be available,
we thus restrict to triplets of points.
A low dimensional embedding of points $y_1,..,y_n$ then will similarly induce weights $\nu^Y$ on triangles, which may be computed by the same formula above, using distances in the low-dimensional space. 
We then may formulate a triplet cross-entropy loss similar to UMAP:
\begin{equation}
CE(X,Y) = \sum_{i,j,k} \mu([x_i,x_j,x_k]) \frac{ \mu([x_i,x_j,x_k])}{\nu^Y([x_i,x_j,x_k])} + (1-\mu([x_i,x_j,x_k])) \frac{1- \mu([x_i,x_j,x_k])}{1-\nu^Y([x_i,x_j,x_k])}.
\end{equation}
As we have seen before, this corresponds to the assumption of independence of individual triangles in the distribution. Note that alternatively, one could use the full KL-divergence, which would also include edge in the loss term.
For the $\phi$-function we use $\phi(t) = \frac{1}{1 + t^2}$, related to the student t distribution and the log-logistic distribution.

In practice, sampling all possible triplets $(i,j,k)$ from the dataset may be prohibitively expensive (as the size of all possible triplets is $\binom{N}{3}$) and could also be uninformative about local structure. Thus, we instead propose to sample 'positive' and 'negative' examples (this is also done in UMAP and may be interpreted as a contrastive estimation scheme \cite{damrich2022t}).
The underlying assumption here is that for positive examples, the weight is close to $1$, such that we only have to compute the first part of the two summands in the loss, while for negative examples, the weight is approximately $0$, meaning only the second term contributes.
A positive example is a local triplet, that is we first sample an edge $(i,j)$ where $j \in \mathcal{N}(i)$, meaning $j$ is a nearest neighbour of $i$. Then, we sample a third point $k$ from the union of neighbourhoods $\mathcal{N}(i) \cup \mathcal{N}(j)$. The three points sampled in this way should thus provide information about local structure to the embedding.
Correspondingly, $n_{\text{negative per positive}}$ negative examples may then simply be sampled uniformly among all possible triplets. Alternatively, one may also sample negative examples which are semi-local, that is, where $(i,j)$ are neighbors and $k$ is then sampled outside of the respective neighborhood. In practice, we mix both these sampling strategies with a proportion of $0.5$ - this corresponds to over-emphasizing semi-local triplets.
See \autoref{algo:cumap} for a pseudocode summary.  All code may be found at \url{https://github.com/jakeck1/cech-umap/}.
\begin{algorithm}
\SetAlgoLined
\KwData{Dataset $X = (x_1, \dots, x_n) \subset \mathbb{R}^d$; number of neighbors $k$; embedding dimension $d_o$; number of epochs $T$,hyperparameters}
\KwResult{Low-dimensional embedding $Y = (y_1, \dots, y_n) \subset \mathbb{R}^{d_o}$}

\BlankLine
\textbf{Compute} $k$-nearest neighbors: $\mathcal{N}(i) \gets \text{knn}(X, k)$ for all $i$\;
\textbf{Initialize} embeddings $Y$ via PCA or randomly\;

\BlankLine
\For{$t = 1, \dots, T$}{
    \For{each mini-batch}{
        \BlankLine
        \textbf{Sample positive triplet:}\\
        \Indp
        Sample $i \sim \mathrm{Unif}(\{1,\dots,n\})$\;
        Sample $j \sim \mathrm{Unif}(\mathcal{N}(i))$\;
        Sample $k \sim \mathrm{Unif}(\mathcal{N}(i) \cup \mathcal{N}(j))$\;
        \Indm

        \BlankLine
        \textbf{Compute weights in input space:}\\
        $\mu_{ijk} \gets 1 - \phi\!\left(r(x_i, x_j, x_k)\right)$, where $r(x_i,x_j,x_k)$ is the minimal enclosing ball radius:
        \[
        r(x_i,x_j,x_k) =
        \begin{cases}
        \tfrac{1}{2} d_{\max}, & \text{if obtuse or right triangle},\\[4pt]
        \dfrac{d_{ij} d_{jk} d_{ki}}{4 \sqrt{s(s-d_{ij})(s-d_{jk})(s-d_{ki})}}, & \text{else},
        \end{cases}
        \]
        with $s = \tfrac{1}{2}(d_{ij} + d_{jk} + d_{ki})$ and $d_{\max} = \max(d_{ij}, d_{jk}, d_{ki})$\;

        \BlankLine
        \textbf{Compute weights in embedding space:}\\
        $\nu_{ijk} \gets 1 - \phi\!\left(r(y_i, y_j, y_k)\right)$\;

        \BlankLine
        \textbf{Compute positive triplet cross-entropy loss:}\\
        \[
        L_{ijk} = -  \mu_{ijk} \log \nu_{ijk}         \]
   
        \BlankLine
        \textbf{Sample } $n_{\text{negative per positive}}$ \textbf{ negative triplets:} with probability $0.5$ sample uniformly $(i,j,k)$ from $\{1,\dots,n\}$\;
        else sample $(i,j)$ as neigbours and $k \notin \mathcal{N}(i) \cup \mathcal{N}(j)$;
        \BlankLine
	\textbf{Compute negative triplet cross-entropy loss}:\\
	
	\[L_{ijk} = - (1-\mu_{ijk})  \log(1 - \nu_{ijk}) \]
        \BlankLine
        \textbf{Update embeddings:} perform gradient step on total loss $L = \sum L_{ijk}$ via autograd\;
    }
}
\caption{\v{C}ech-UMAP pipeline}\label{algo:cumap}
\end{algorithm}

Above we have omitted the issue of rescaling the distances, which is performed in UMAP to obtain density-scaled local metrics. 
Indeed, for edges this would be straightforward: denote $ad_{\mathbb{R}^d}$ the rescaled euclidean distance by a constant factor $a$.
Then, the smallest radius $r$ such that the two rescaled balls $B_{r, a d_{\mathbb{R}^d}}(x), B_{r, b d_{\mathbb{R}^d}}(y)$ intersect is given by 
\begin{equation}
r^* = \frac{ab}{a+b} d_{\mathbb{R}^d}(x,y).
\end{equation}
In particular, if $a = \frac{1}{d(\text{knn}(x),x)}, b= \frac{1}{d(\text{knn}(y),y)}$ is division by the distance to the $k$-nearest neighbor (that is, density adjusted rescaling), then this simply corresponds to a rescaling by the sum of these distances.
However, for triangles, no such simple closed form characterization under rescaling is available. Although one could possibly derive an approximation, we eschew this issue and omit the local rescaling. Instead, we simply introduce a global scale factor, as the maximum distance of any point in the dataset to its $k$-nearest neighbor, and divide each $r[x_i,x_j,x_k]$ by this global factor. 
Further implementational details that are taken over from UMAP are: the computation of the k-neighborhoods is approximated by pynndescent, and the weights for negative triplets  in the cross-entropy are binarized. 

\subsubsection{Results}
We apply \v{C}UMAP to a number of standard high-dimensional datasets.
These are meant to exhibit whether the method may extract any meaningful known structure in the low-dimensional embeddings, here qualitatively evaluated by ground truth additional information about the datapoints in terms of their membership in certain classes (if classes are known) or clusters (if no classes are provided). 
For example, in the MNIST dataset, a method might separate datapoints that belong to different digits in distinct, well-formed clusters in $2$-dimensional space. 
See the Appendix for a description of the datasets used.
In general, we find that the overall structure of the embeddings, at least for certain numbers of neighbors, are often remarkably similar between \v{C}UMAP and UMAP, see \autoref{fig:embeddings1} and \autoref{fig:embeddings2}. This might possibly be due to the fact that we use a similar $\phi$-function as is used in UMAP, which has been noted before to be the main driver of the embeddings together with the cross entropy loss. Another possibility is that the embedding is mainly driven by obtuse triangles - i.e., those where where the scale in the filtration is a function of the maximal edge length alone - hence similar to UMAP effectively inducing a weight on edges.
To obtain a quantitative comparison of the two methods, we evaluate both according to three complementary metrics, each designed to capture a distinct notion of structure preservation.

First, we compute the trustworthiness of the embedding \citep{venna2001neighborhood}. The trustworthiness quantifies how well local neighborhoods in the high-dimensional space are preserved in the embedding. Formally, for an embedding 
$Y=(y_1,…,y_n)$ of data points 
$X=(x_1,…,x_n)$, it is defined as
\begin{equation}
T(k) = 1 - \frac{2}{n k (2n - 3k - 1)} \sum_{i=1}^n \sum_{j \in U^k(y_i)} \big( r(i,j) - k \big),
\end{equation}
where $r(i,j)$ is the rank of point  $x_j$ in the ordered list of distances from $x_i$  in the original space, and 
$ U^k(y_i)$
are the indices of points that are among the
$k$ nearest neighbors of $y_i$  in the embedding but not among the  $k$ nearest neighbors of 
$x_i$ in the original space. Values close to 
$1$ indicate faithful preservation of local neighborhood structure.
Secondly, to assess preservation of global structure, we use the metric proposed in \cite{amid2019trimap}, which measures how well the embedding aligns with a linearly optimal embedding such as PCA.  As PCA is often interpreted to reflect global structure, this thus gives an estimate of how close the method comes to the global structure preservation of PCA (this of course hinges on how well PCA is able to represent global structure in the first place). 
Specifically, given the PCA embedding 
$Y_{\text{PCA}}$and another, centered, embedding 
$Y$, one computes a normalized Procrustes correlation:
\begin{equation}
G = 1 - \frac{\min_{R \in O(d)} |\vert \vert Y - Y_{\text{PCA}} R \vert \vert_F}{\vert \vert Y_{\text{PCA}} \vert \vert_F},
\end{equation}
where the minimization is over all orthogonal transformations 
$R \in O(d)$, and $\vert \vert \cdot \vert \vert_F$ is the Frobenius norm.
The PCA embedding, by construction, achieves $G=1$. 
Other embeddings attain values closer to 
$1$, the closer they are to the PCA embedding and hence the better they preserve global relationships among points in the sense that the former does.
Lastly, to evaluate topological preservation, we use tools from persistent homology. Specifically, we subsample 
$n=800$
points for computational feasibility and compute Vietoris–Rips filtrations up to the first homology group 
$H_1$
 for both the original and embedded data. The resulting persistence diagrams $D_X$ and  $D_Y$ are then compared using the 2-Wasserstein distance:
\begin{equation}
W_2(D_X, D_Y) = \left( \inf_{\gamma} \sum_{(p,q)\in\gamma} | p - q |_2^2 \right)^{1/2},
\end{equation}
where  $\gamma$ ranges over all bijections between the two diagrams (allowing matches with the diagonal). Smaller Wasserstein distances indicate that the topological features—such as connected components, loops, and voids—are better preserved in the embedding. To increase robustness, we compute the distance for $K = 30$ different subsamples and take the average. Furthermore, to account for variations in the embeddings themselves due to inherent stochasticity of the process we take the mean over $5$ embeddings with different seeds for each of the metrics.

Together, these three metrics offer a complementary view: trustworthiness captures local fidelity, Procrustes alignment captures global geometry, and persistent homology captures topological structure.
As may be observed in \autoref{fig:merged1} and \autoref{fig:merged2}, in general we find the following structure on the datasets tested here: \v{C}UMAP always outperforms in the global metric - this is consistent with \cite{amid2019trimap}, which similiarly used a PCA-initialized method based on triplets which outperformed in this metric. Regarding trustworthiness, for small numbers of nearest neighbors UMAP always outperforms, while for higher numbers of neighbors the results are mixed, with a general tendency for lower scores in both. The topological distances are mixed, with \v{C}UMAP sometimes achieving slightly lower distances, especially when increasing the number of neighbors, while UMAP often increases distances there. In particular, the distance in $H_0$, which might indicate how well the structure of connected components or clusters is preserved is then sometimes slightly better. The distance for $H_1$, which points at circular features, is for example relevant in the COIL20 dataset, where we can see that \v{C}UMAP outperforms the other methods. Thus, we may indeed conclude that quantitatively, \v{C}UMAP seems to be meaningfully representing topological and global structure in the embeddings. Overall, the topological distances are similar in both methods.
Qualitatively, we note however that the embeddings are often less smooth with more spurious points \autoref{fig:embeddings1}, \autoref{fig:embeddings2}. As our goal in this article is to expose the theory of fuzzy simplicial sets and how they may guide algorithmic design, we leave improvements of such defects open for future work.
\begin{figure}
  \centering
  \includegraphics[width=\textwidth]{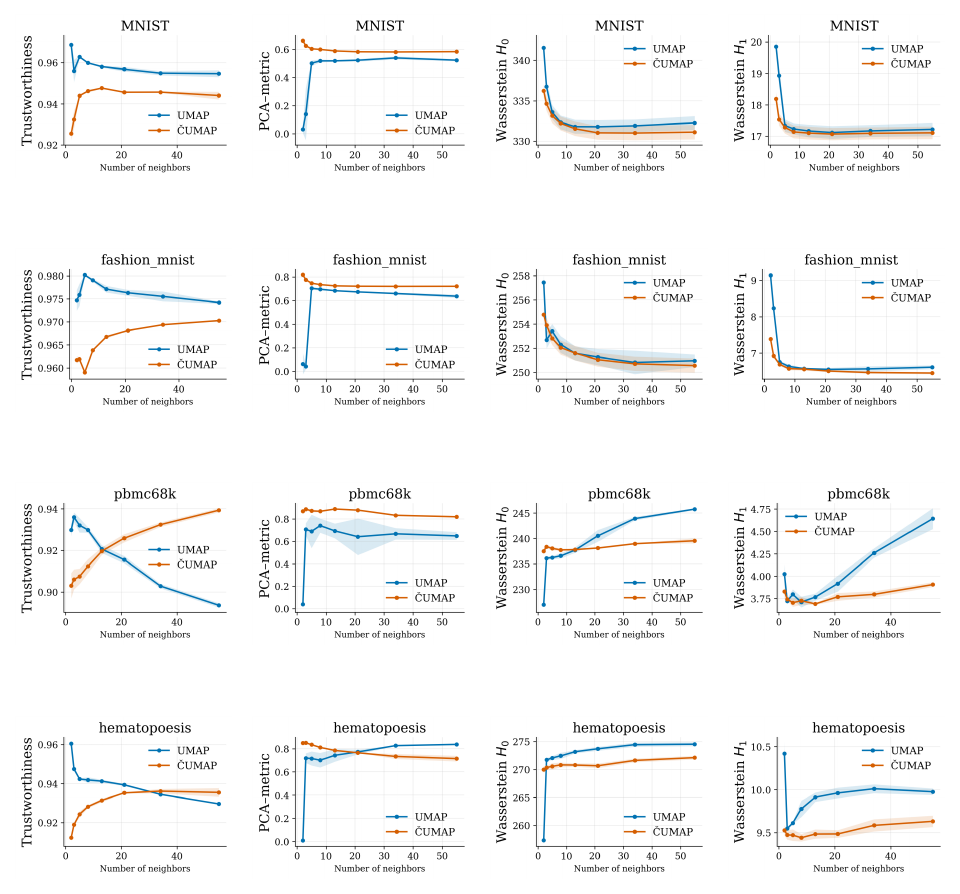}
  \caption{Comparison of UMAP and \v{C}UMAP on various datasets. The methods are evaluated on how well they preserve local structure (trustworthiness, closer to 1 is better) how well they capture global structure  (PCA-metric, closer to 1 is better) and how well they preserve topological features (Wasserstein distances, lower is better).}
  \label{fig:merged1}
\end{figure}

\begin{figure}
  \centering
  \includegraphics[width=\textwidth]{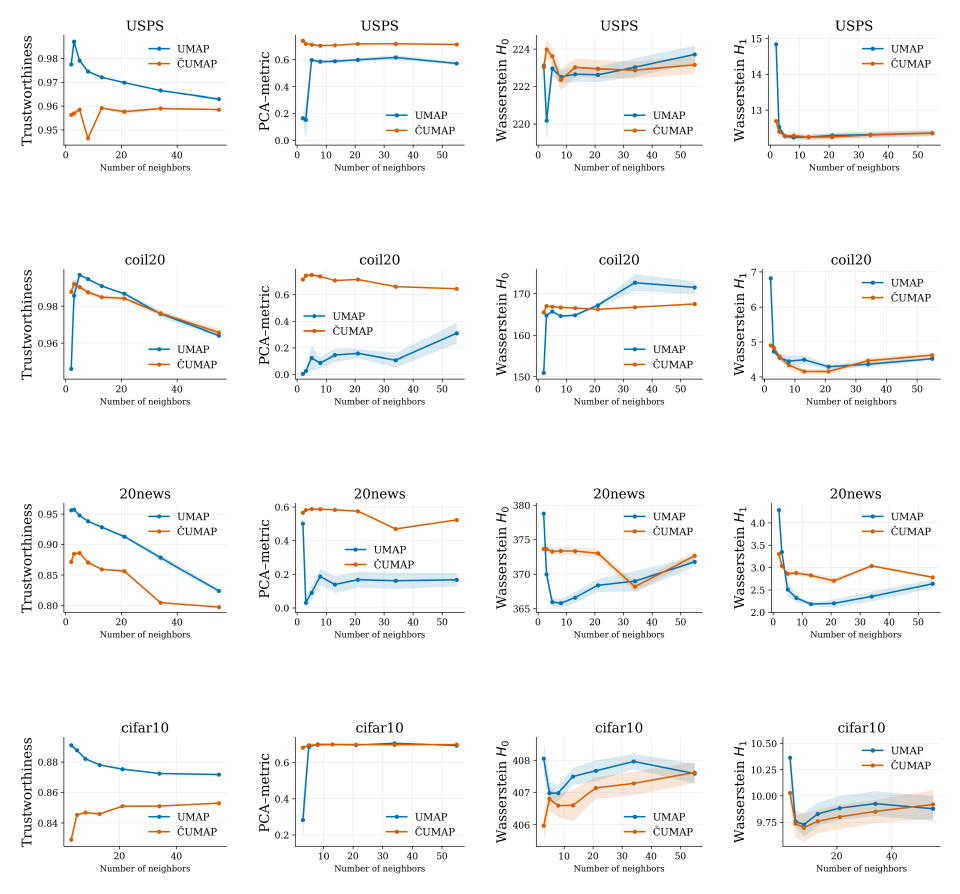}
  \caption{Same evaluation as in \autoref{fig:merged2} for additional datasets.}
  \label{fig:merged2}
\end{figure}

\begin{figure}
  \centering
  \includegraphics[width=\textwidth]{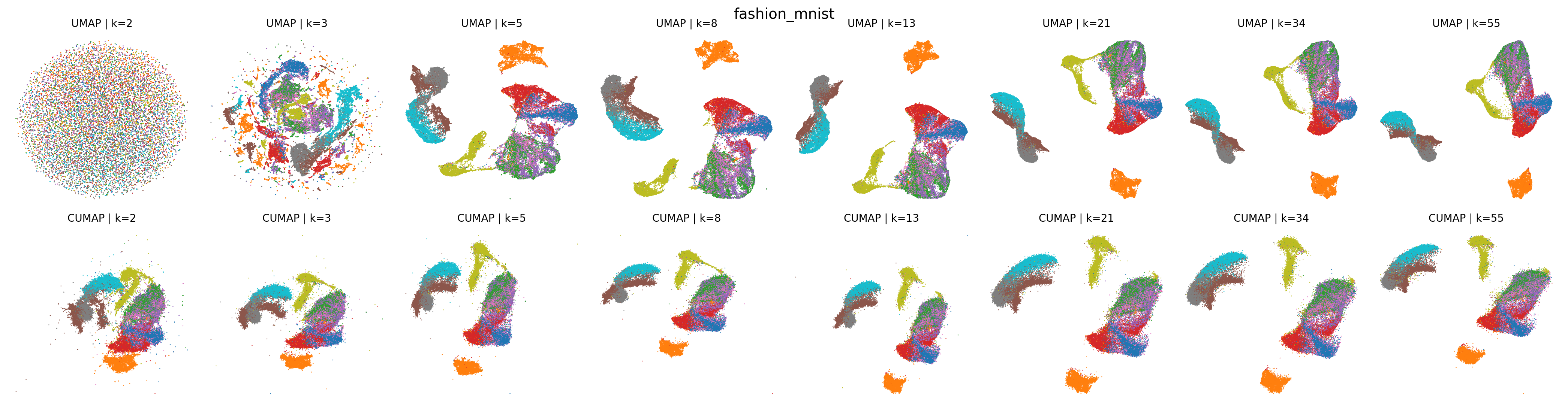}
  \includegraphics[width=\textwidth]{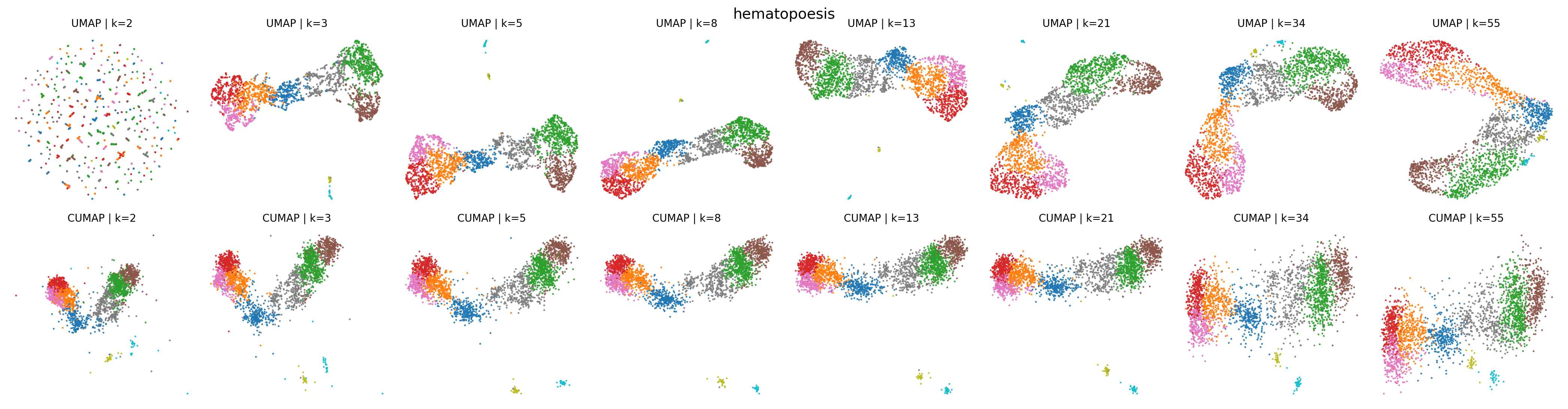}
  \includegraphics[width=\textwidth]{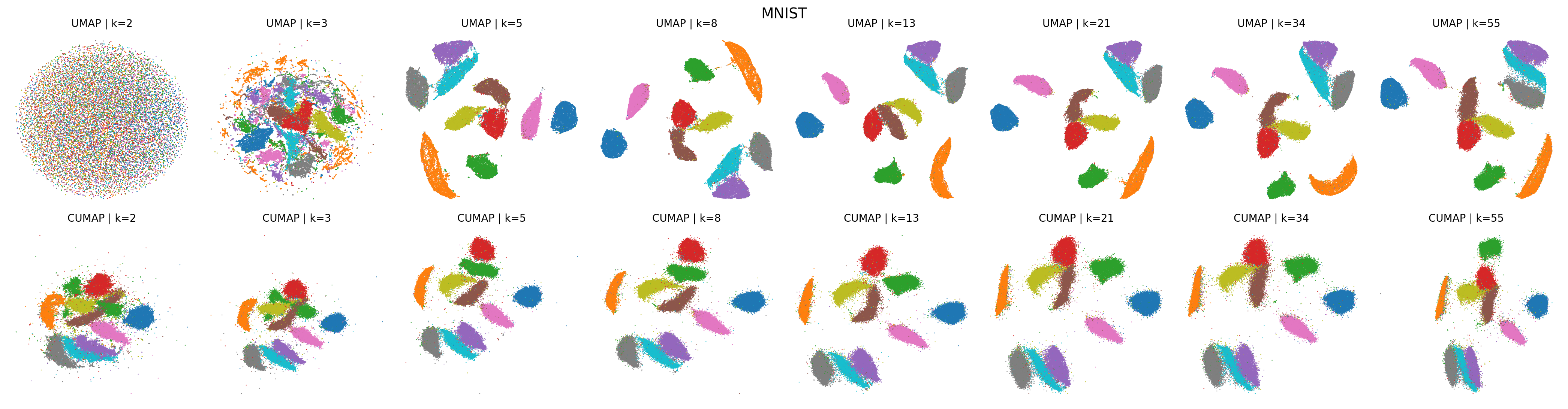}
  \includegraphics[width=\textwidth]{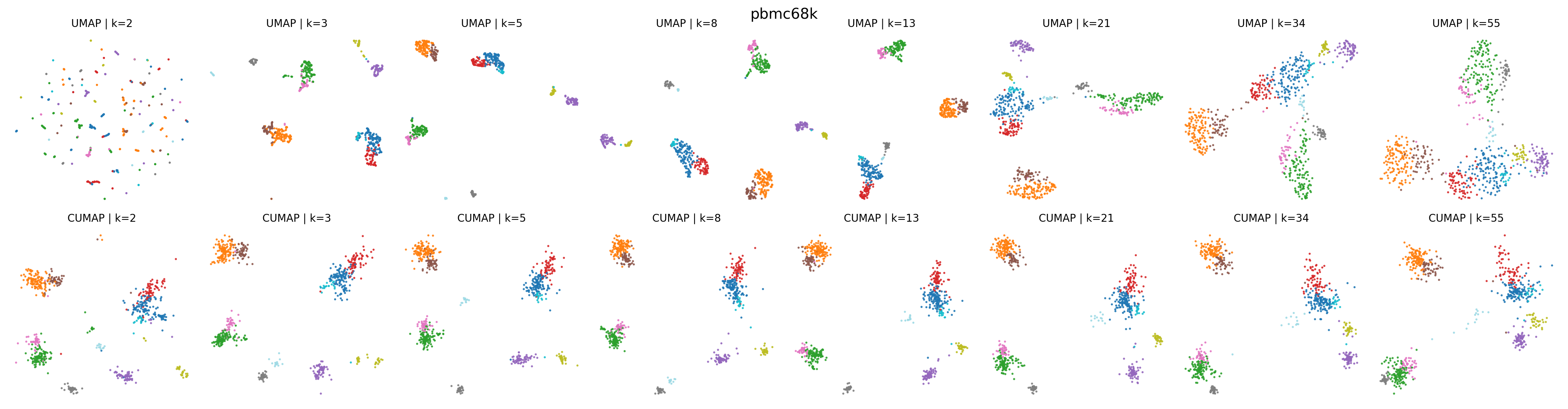}
  \caption{Embedding generated by UMAP and \v{C}UMAP for datasets evaluated in \autoref{fig:merged1}}
  \label{fig:embeddings1}
\end{figure}
\begin{figure}
  \centering
  \includegraphics[width=\textwidth]{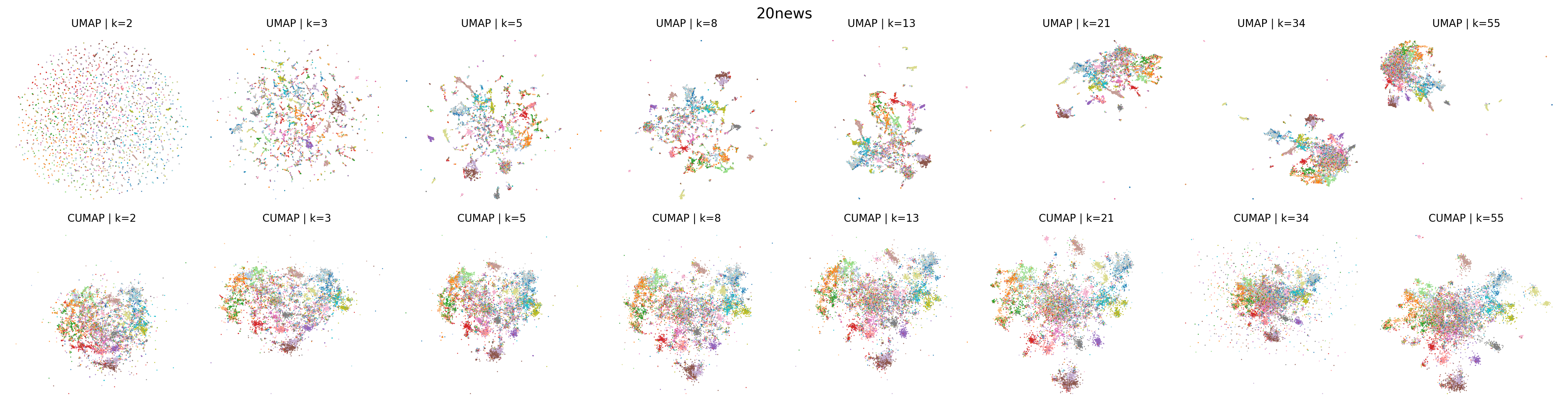}
    \includegraphics[width=\textwidth]{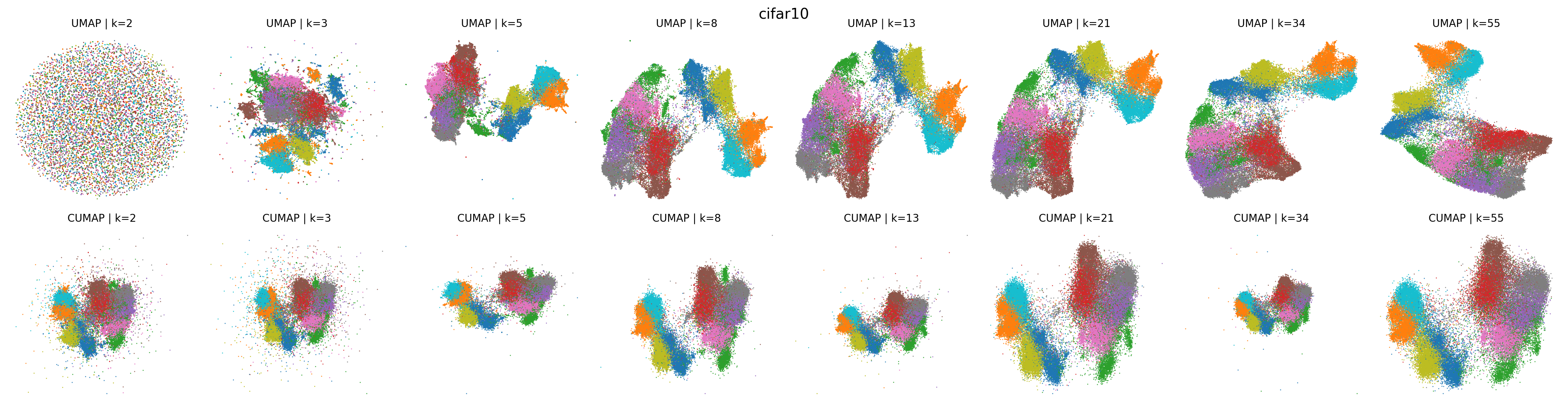}
  \includegraphics[width=\textwidth]{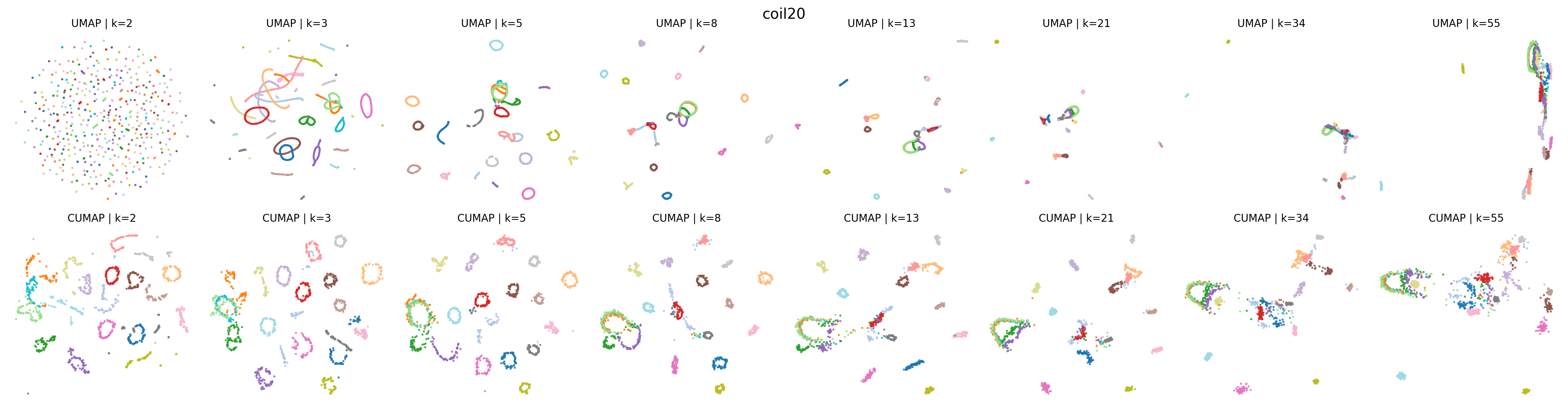}
  \includegraphics[width=\textwidth]{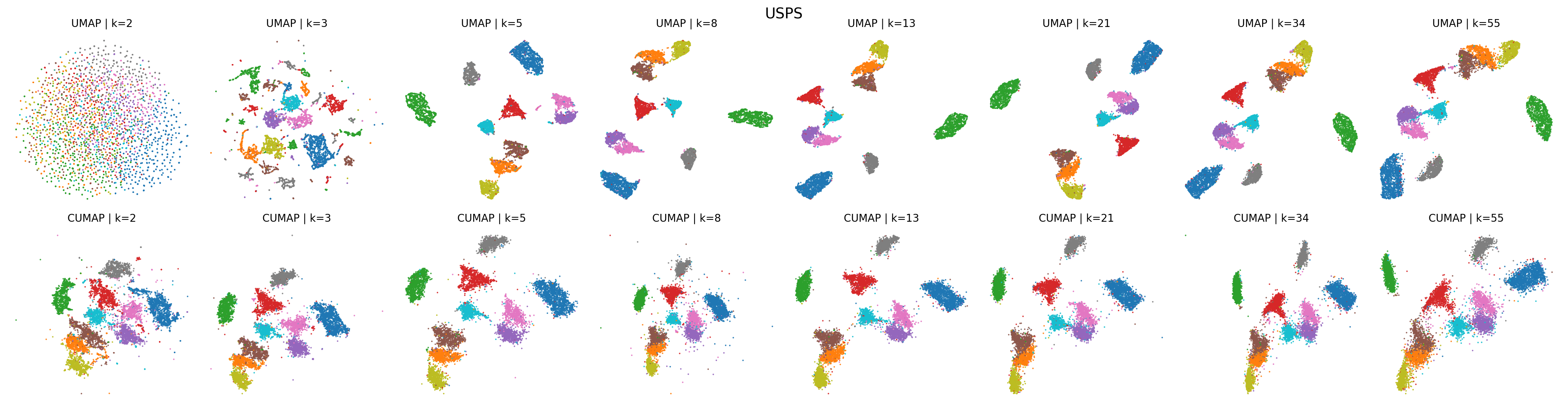}
  \caption{Embedding generated by UMAP and \v{C}UMAP for datasets evaluated in \autoref{fig:merged2}}
  \label{fig:embeddings2}
\end{figure}

We however note that one possible modification is to use the intrinsic \v{C}ech filtration instead of the extrinsic one. 
That is, determining the minimal scale $R$ at which a triangle appears as 
\begin{equation}
r(x_i,x_j,x_k)=  \min_{x_l}  \max (d(x_i,x_l), d(x_j,x_l),d(x_k,x_k) ).  
\end{equation}
Since taking the minimum over all points is prohibitively costly and checks many points that are probably not relevant, one may restrict to the union of neighborhoods, furthermore hard maxima/minima may be replaced by soft versions, that is 
\begin{equation}
r(x_i,x_j,x_k) = \operatorname{softmin}_{x_l \in \mathcal{N}(x_i) \cup \mathcal{N}(x_j) \cup \mathcal{N}(x_k)}   \operatorname{softmax} (d(x_i,x_l), d(x_j,x_l),d(x_k,x_k) ).  
\end{equation}
We observationally find that constructing weights in this way yields similar, if somewhat smoother embeddings than the previous version. We include this construction as an option in the released code.

\subsubsection{Curvature Complexes}
We may also observe that the simplices in the (intrinsic) \v{C}ech-filtration are defined by a condition very similar to the one used to define curvature in metric spaces as described in \cite{joharinad2019topology}. Given a metric space $(X,d)$, the curvature $\rho:X\times X\times X \to \mathbb{R}_{\ge 0}$ is defined by

\begin{equation}
  \label{eq:simplerCurvature}
  \begin{split}
    \rho(x_1,x_2,x_3) =  \inf_{x\in X} \max_{k\in \{1,2,3\}} \left\{ \frac{d(x_k,x)}{r_k} ~\bigg|~ 
    \begin{pmatrix}
      r_1+r_2 = d(x_1,x_2)\\
      r_2+r_3 = d(x_2,x_3)\\
      r_3+r_1 = d(x_3,x_1)
    \end{pmatrix}
     \right\}.
  \end{split}
\end{equation}
Note that the constraint consists of 3 equations for 3 unknowns and can always be solved. The solutions for $r_1,r_2,r_3$ are called the Gromov products. 

We can now observe that, up to the division by the Gromov products, the definition of curvature agrees with the term $\min_{x \in X} \max_{k\in\{1,\cdots,n \}} d(x_{i_k},x)$ in the definition of the \v{C}ech filtration. This gives rise to the idea that one could define a curvature complex instead.
 To do so, we must first generalize the notion of metric curvature to $n$ points. A possible approach is to consider
 \begin{equation}
  \begin{split}
    \rho([x_{i_0},\cdots,x_{i_n}]):= \inf_{x \in X} \max_{r_k>0, k\in\mathbf{n}} \left\{ \frac{d(x_{i_k},x)}{r_k} ~\bigg|~ r_l + r_{m} \ge d(x_{i_l},x_{i_m})~ \forall l,m\in \mathbf{n}\right\}.
  \end{split}
  \label{eq:simplexCurvature-alt}
\end{equation}
We can then  define a \emph{\v{C}ech curvature filtration} via the weights
\begin{equation}
  \begin{split}
    \mu^r_R\left([x_{i_0},...,x_{i_n}] \right)=
    \begin{cases}
      1,&n=1,\\
      \delta\left(\rho([x_{i_0},\cdots,x_{i_n}]) \leq r\right),&n>1.
    \end{cases}
  \end{split}
\end{equation}
and the corresponding sets of the filtration by $R(X,r):=(\mu^r_R)^{-1}(1)$.
We can then compare this construction that encodes \emph{higher-order curvature} with  the \v{C}ech filtration that encodes the distance information in the space. The \v{C}ech complex tells us when distance balls have a common intersection. The curvature complex tells us how much radii of balls that are sufficiently large to enable pairwise intersections have to be enlarged to get a joint intersection of all these balls. We recall that the Vietoris-Rips complex automatically fills in a simplex when the balls intersect pairwise. In that sense, the curvature complex tells us about the difference between the Vietoris-Rips and the \v{C}ech complex.\\
Similar to how fuzzy simplicial sets were obtained from a \v{C}ech filtration in Section \ref{sec:umapCech}, we can then define weights for fuzzy simplicial sets using eq.~\eqref{eq:simplexCurvature-alt}:
\begin{equation}
  \begin{split}
    \psi([x_{i_0},\cdots,x_{i_n}]) := 1 -
    \begin{cases}
      \phi(1),&n=1,\\
      \phi(\rho([x_{i_0},\cdots,x_{i_n}])),& n>1.
    \end{cases}
  \end{split}
\end{equation}
where $\phi$ is some cumulative distribution function that serves as distances-to-weights function.

One could now in principle use a force-directed graph layout, as in UMAP, or another embedding method, to embed the corresponding graph with those weights into some (usually low-dimensional) space. The embedding method would then automatically arrange the embedded points such that they exhibit similar curvature to the points in the original metric space.

\subsection{Methods Based on the Rank Order of Distances}
As an alternative avenue for modification, one may note that the
distributions over Vietoris-Rips filtrations have a Markovian
structure among their edges.
\begin{proposition}
Let $X = (x_1,\dots,x_n)$ be a sample from a metric space.
Put a total order on the edges $[x_i,x_j]$ by declaring
\[
  [x_i,x_j] \le [x_k,x_l]
  \quad\Longleftrightarrow\quad
  d(x_i,x_j) \le d(x_k,x_l).
\]
Form a line graph whose vertices are these edges in sorted order, and
where two consecutive edges are connected if and only if no other edge
$[x_r,x_s]$ has distance strictly between them.
Then the distribution
\[
  p(S)
  \;=\;
  \int \delta_{VR(X,r)}(S)\, p(r)\, dr
\]
over edges is Markovian with respect to this graph.  
That is, if
\[
  d(x_{i_k},x_{j_k})
  \;\le\;
  d(x_{i_{k+1}},x_{j_{k+1}})
\]
for $k=0,\dots,N-1$, then
\small{
\begin{equation} \begin{split}
 &p\!\left(
   S[x_{i_N},x_{j_N}] = s_N,
   S[x_{i_{N-1}},x_{j_{N-1}}] = s_{N-1},
   \dots,
   S[x_{i_1},x_{j_1}] = s_1,
   S[x_{i_0},x_{j_0}] = s_0
 \right)
 \\[0.3em]
 \qquad=
 &p\!\left(
    S[x_{i_N},x_{j_N}] = s_N
    \,\middle|\,
    S[x_{i_{N-1}},x_{j_{N-1}}] = s_{N-1}
 \right)
 \nonumber\\[-0.3em]
 \qquad\quad\times
 &p\!\left(
   S[x_{i_{N-1}},x_{j_{N-1}}] = s_{N-1}
   \,\middle|\, x_{i_{N-2}},x_{j_{N-2}}] = s_{N-2}
    \right)
 \cdots \\
 \times &p\!\left(
   S[x_{i_1},x_{j_1}] = s_1
   \,\middle|\,
   S[x_{i_0},x_{j_0}] = s_0
 \right)
 p\!\left(S[x_{i_0},x_{j_0}] = s_0\right).
 \nonumber
\end{split} \end{equation}}
\end{proposition}

\begin{proof}
This follows directly from the formulas in \autoref{ex:VR}.  
Intuitively: once we know whether the next smaller edge is present in
the Vietoris--Rips complex, the presence or absence of all strictly smaller
edges provides no additional information.
\end{proof}

If one posits a similar conditional structure for $q^Y$, this suggests
an alternative factorization of the Kullback--Leibler divergence:

\small{
\begin{equation} \begin{split}
\mathrm{D_{KL}}(p\Vert q)
=
\sum_{i}
\Big[
 \, &p\!\left(
   S[x_{i_N},x_{j_N}] = 1,\;
   S[x_{i_{N-1}},x_{j_{N-1}}] = 1
 \right)
 \ln q^Y\!\left(
   S[y_{i_N},y_{j_N}] = 1
   \,\middle|\,
   S[y_{i_{N-1}},y_{j_{N-1}}] = 1
 \right)
 \nonumber\\
 \quad\;+
 &p\!\left(
   S[x_{i_N},x_{j_N}] = 0,\;
   S[x_{i_{N-1}},x_{j_{N-1}}] = 1
 \right)
 \ln q^Y\!\left(
   S[y_{i_N},y_{j_N}] = 0
   \,\middle|\,
   S[y_{i_{N-1}},y_{j_{N-1}}] = 1
 \right)
 \nonumber\\
 \quad\;+
 &p\!\left(
   S[x_{i_N},x_{j_N}] = 0,\;
   S[x_{i_{N-1}},x_{j_{N-1}}] = 0
 \right)
 \ln q^Y\!\left(
   S[y_{i_N},y_{j_N}] = 0
   \,\middle|\,
   S[y_{i_{N-1}},y_{j_{N-1}}] = 0
 \right)
 \Big].
\end{split} \end{equation}}
This follows by applying \autoref{lem:factorizeKL} to the above DAG and
observing that \[
  p\!\left(
    S[x_{i_N},x_{j_N}] = 1,\;
    S[x_{i_{N-1}},x_{j_{N-1}}] = 0
  \right)
  = 0.
\]
For an appropriate definition of $q^Y$, this encourages the learned
low-dimensional points $Y$ to preserve the same distance order structure
as present in $X$.
The connection of such a loss to ordinal embeddings \citep{vankadara2023insights} or non-metric-multidimensional scaling \citep{kruskal1964nonmetric}, where also the preservation of rank-order is the objective, might be an interesting direction for further work.
We also note that \cite{amid2019trimap} is an embedding method based on triplets which samples triplets $(i,j,k)$ where point $j$ is closer to point $i$ than point $k$, and enforces this order structure in the low-dimensional embeddings. Thus, one may see this as an approximation where instead of maintaining global rank order of distances one independently maintains rank order of distances from each point. 
 
\section{Discussion}
We have introduced a probabilistic framework that is able to represent all fuzzy simplicial sets as objects generated from probability distributions over classic simplicial sets. We have studied operations for merging and comparing such objects, and have studied simple examples from filtration.  In particular, we have used these examples to obtain a probabilistic interpretation of the loss of UMAP, based on Vietoris-Rips filtrations. 
By recasting fuzzy simplicial sets in probabilistic terms, we hope to make the underlying machinery more accessible to a broader audience and to facilitate methodological extensions. In this spirit, we explored how generative models over simplicial sets may serve as a basis for new embedding procedures. Our method, ČUMAP, provides an initial demonstration of this idea by producing UMAP-like embeddings using a triplet-based objective.
Several natural directions now follow from this viewpoint. One immediate extension is to combine the triplet-based loss with the edge-weighting scheme of UMAP, which the probabilistic formulation developed here accommodates directly. More broadly, we believe that interpreting fuzzy simplicial sets through the lens of probability offers a flexible foundation on which future variants of UMAP and related manifold-learning methods can be built.

\bibliography{prob_bib}

@article{embrechts2013note,
  title={A note on generalized inverses},
  author={Embrechts, Paul and Hofert, Marius},
  journal={Mathematical Methods of Operations Research},
  volume={77},
  pages={423--432},
  year={2013},
  publisher={Springer}
}

@article{mcinnes2018umap,
  title={Umap: Uniform manifold approximation and projection for dimension reduction},
  author={McInnes, Leland and Healy, John and Melville, James},
  journal={arXiv preprint arXiv:1802.03426},
  year={2018}
}

@article{ghojogh2021uniform,
  title={Uniform manifold approximation and projection (UMAP) and its variants: tutorial and survey},
  author={Ghojogh, Benyamin and Ghodsi, Ali and Karray, Fakhri and Crowley, Mark},
  journal={arXiv preprint arXiv:2109.02508},
  year={2021}
}

@article{diaz2021review,
  title={A review of UMAP in population genetics},
  author={Diaz-Papkovich, Alex and Anderson-Trocm{\'e}, Luke and Gravel, Simon},
  journal={Journal of Human Genetics},
  volume={66},
  number={1},
  pages={85--91},
  year={2021},
  publisher={Springer Singapore Singapore}
}

@article{joharinad2019topology,
  title={Topology and curvature of metric spaces},
  author={Joharinad, Parvaneh and Jost, J{\"u}rgen},
  journal={Advances in mathematics},
  volume={356},
  pages={106813},
  year={2019},
  publisher={Elsevier}
}

@misc{barth2024fuzzysimplicialsetsapplication,
      title={Fuzzy simplicial sets and their application to geometric data analysis}, 
      author={Lukas Silvester Barth and Fatemeh and Fahimi and Parvaneh Joharinad and Jürgen Jost and Janis Keck and Thomas Jan Mikhail},
      year={2024},
      eprint={2406.11154},
      archivePrefix={arXiv},
      primaryClass={math.AT},
      url={https://arxiv.org/abs/2406.11154}, 
}

@article{spivak2009metric,
  title={Metric realization of fuzzy simplicial sets},
  author={Spivak, David I},
  journal={Preprint},
  volume={4},
  year={2009}
}

@article{vankadara2023insights,
  title={Insights into ordinal embedding algorithms: A systematic evaluation},
  author={Vankadara, Leena Chennuru and Lohaus, Michael and Haghiri, Siavash and Wahab, Faiz Ul and Von Luxburg, Ulrike},
  journal={Journal of Machine Learning Research},
  volume={24},
  number={191},
  pages={1--83},
  year={2023}
}

@article{kruskal1964nonmetric,
  title={Nonmetric multidimensional scaling: a numerical method},
  author={Kruskal, Joseph B},
  journal={Psychometrika},
  volume={29},
  number={2},
  pages={115--129},
  year={1964},
  publisher={Springer-Verlag New York}
}

@article{wasserman2018topological,
  title={Topological data analysis},
  author={Wasserman, Larry},
  journal={Annual Review of Statistics and Its Application},
  volume={5},
  number={1},
  pages={501--532},
  year={2018},
  publisher={Annual Reviews}
}

@article{friedman2012survey,
  title={Survey article: an elementary illustrated introduction to simplicial sets},
  author={Friedman, Greg},
  journal={The Rocky Mountain Journal of Mathematics},
  pages={353--423},
  year={2012},
  publisher={JSTOR}
}

@article{zadeh1968probability,
  title={Probability measures of fuzzy events},
  author={Zadeh, Lotfi Asker},
  journal={Journal of mathematical analysis and applications},
  volume={23},
  number={2},
  pages={421--427},
  year={1968},
  publisher={Academic Press}
}

@article{singpurwalla2004membership,
  title={Membership functions and probability measures of fuzzy sets},
  author={Singpurwalla, Nozer D and Booker, Jane M},
  journal={Journal of the American statistical association},
  volume={99},
  number={467},
  pages={867--877},
  year={2004},
  publisher={Taylor \& Francis}
}

@article{hirota1981concepts,
  title={Concepts of probabilistic sets},
  author={Hirota, Kaoru},
  journal={Fuzzy sets and systems},
  volume={5},
  number={1},
  pages={31--46},
  year={1981},
  publisher={Elsevier}
}

@article{zadeh1965fuzzy,
  title={Fuzzy sets},
  author={Zadeh, Lotfi A},
  journal={Information and Control},
  year={1965}
}

@book{zimmermann2011fuzzy,
  title={Fuzzy set theory—and its applications},
  author={Zimmermann, Hans-J{\"u}rgen},
  year={2011},
  publisher={Springer Science \& Business Media}
}

@article{jansma2025mereological,
  title={Mereological approach to higher-order structure in complex systems: From macro to micro with M{\"o}bius},
  author={Jansma, Abel},
  journal={Physical Review Research},
  volume={7},
  number={2},
  pages={023016},
  year={2025},
  publisher={APS}
}

@article{rudin1991functional,
  title={Functional analysis 2nd ed},
  author={Rudin, Walter},
  journal={International Series in Pure and Applied Mathematics. McGraw-Hill, Inc., New York},
  pages = {75},
  year={1991}
}

@article{shiebler2021flattening,
  title={Flattening multiparameter hierarchical clustering functors},
  author={Shiebler, Dan},
  journal={arXiv preprint arXiv:2104.14734},
  year={2021}
}

@article{shiebler2020functorial,
  title={Functorial manifold learning},
  author={Shiebler, Dan},
  journal={arXiv preprint arXiv:2011.07435},
  year={2020}
}

@article{sainburg2021parametric,
  title={Parametric UMAP embeddings for representation and semisupervised learning},
  author={Sainburg, Tim and McInnes, Leland and Gentner, Timothy Q},
  journal={Neural Computation},
  volume={33},
  number={11},
  pages={2881--2907},
  year={2021},
  publisher={MIT Press One Rogers Street, Cambridge, MA 02142-1209, USA journals-info~…}
}

@article{damrich2022t,
  title={From $ t $-SNE to UMAP with contrastive learning},
  author={Damrich, Sebastian and B{\"o}hm, Jan Niklas and Hamprecht, Fred A and Kobak, Dmitry},
  journal={arXiv preprint arXiv:2206.01816},
  year={2022}
}

@article{damrich2021umap,
  title={On UMAP's true loss function},
  author={Damrich, Sebastian and Hamprecht, Fred A},
  journal={Advances in Neural Information Processing Systems},
  volume={34},
  pages={5798--5809},
  year={2021}
}

@article{jardine2020stability,
  title={Stability for UMAP},
  author={Jardine, JF},
  journal={arXiv preprint arXiv:2011.13430},
  year={2020}
}

@article{draganov2023actup,
  title={Actup: Analyzing and consolidating tsne and UMAP},
  author={Draganov, Andrew and J{\o}rgensen, Jakob R{\o}dsgaard and Nellemann, Katrine Scheel and Mottin, Davide and Assent, Ira and Berry, Tyrus and Aslay, Cigdem},
  journal={arXiv preprint arXiv:2305.07320},
  year={2023}
}

@article{ravuri2024towards,
  title={Towards One Model for Classical Dimensionality Reduction: A Probabilistic Perspective on UMAP and t-SNE},
  author={Ravuri, Aditya and Lawrence, Neil D},
  journal={arXiv preprint arXiv:2405.17412},
  year={2024}
}

@book{koller2009probabilistic,
  title={Probabilistic graphical models: principles and techniques},
  author={Koller, Daphne and Friedman, Nir},
  year={2009},
  publisher={MIT press}
}

@article{stanley1986two,
  title={Two poset polytopes},
  author={Stanley, Richard P},
  journal={Discrete \& Computational Geometry},
  volume={1},
  number={1},
  pages={9--23},
  year={1986},
  publisher={Springer}
}

@article{caratheodory1911variabilitatsbereich,
  title={{\"U}ber den Variabilit{\"a}tsbereich der Fourier’schen Konstanten von positiven harmonischen Funktionen},
  author={Carath{\'e}odory, Constantin},
  journal={Rendiconti Del Circolo Matematico di Palermo (1884-1940)},
  volume={32},
  number={1},
  pages={193--217},
  year={1911},
  publisher={Springer}
}

@article{bohm2020unifying,
  title={A unifying perspective on neighbor embeddings along the attraction-repulsion spectrum},
  author={B{\"o}hm, Niklas and Berens, Philipp and Kobak, Dmitry},
  year={2020}
}

@article{amid2019trimap,
  title={TriMap: Large-scale dimensionality reduction using triplets},
  author={Amid, Ehsan and Warmuth, Manfred K},
  journal={arXiv preprint arXiv:1910.00204},
  year={2019}
}

@article{xiao2017fashion,
  title={Fashion-mnist: a novel image dataset for benchmarking machine learning algorithms},
  author={Xiao, Han and Rasul, Kashif and Vollgraf, Roland},
  journal={arXiv preprint arXiv:1708.07747},
  year={2017}
}

@article{lecun1998mnist,
  title={The MNIST database of handwritten digits},
  author={LeCun, Yann},
  journal={http://yann. lecun. com/exdb/mnist/},
  year={1998}
}

@article{zheng2017massively,
  title={Massively parallel digital transcriptional profiling of single cells},
  author={Zheng, Grace X Y and El-Khoury, Rocio J and Feger, Brad J and et al.},
  journal={Nature Communications},
  volume={8},
  number={1},
  pages={14049},
  year={2017},
  publisher={Nature Publishing Group},
  doi={10.1038/ncomms14049}
}

@article{paul2015transcriptional,
  title={Transcriptional heterogeneity and lineage commitment in myeloid progenitors},
  author={Paul, Franziska and Arkin, Ya’ara and Giladi, Amir and Jaitin, Diego Adhemar and Kenigsberg, Ephraim and Keren-Shaul, Hadas and Winter, Deborah and Lara-Astiaso, David and Gury, Meital and Weiner, Assaf and others},
  journal={Cell},
  volume={163},
  number={7},
  pages={1663--1677},
  year={2015},
  publisher={Elsevier}
}

@article{hull2002database,
  title={A database for handwritten text recognition research},
  author={Hull, Jonathan J.},
  journal={IEEE Transactions on pattern analysis and machine intelligence},
  volume={16},
  number={5},
  pages={550--554},
  year={2002},
  publisher={IEEE}
}

@techreport{Nene1996coil,
  author = {Nene, Sameer A. and Nayar, Shree K. and Murase, Hiroshi},
  title = {Columbia Object Image Library (COIL-20)},
  institution = {Columbia University},
  year = {1996},
  number = {CUCS-005-96}
}

@misc{Mitchell1997Twenty,
  author = {Mitchell, T.},
  title = {Twenty Newsgroups},
  year = {1997},
  publisher = {UCI Machine Learning Repository},
  note = {doi: 10.24432/C5C323},
  url = {https://archive.ics.uci.edu/ml/datasets/Twenty+Newsgroups}
}

@TECHREPORT{Krizhevsky09learningmultiple,
  author = {Alex Krizhevsky},
  title = {Learning multiple layers of features from tiny images},
  institution = {University of Toronto},
  year = {2009}
}

@inproceedings{deng2009imagenet,
  title={Imagenet: A large-scale hierarchical image database},
  author={Deng, Jia and Dong, Wei and Socher, Richard and Li, Li-Jia and Li, Kai and Fei-Fei, Li},
  booktitle={2009 IEEE conference on computer vision and pattern recognition},
  pages={248--255},
  year={2009},
  organization={IEEE}
}

@inproceedings{zomorodian2004computing,
  title={Computing persistent homology},
  author={Zomorodian, Afra and Carlsson, Gunnar},
  booktitle={Proceedings of the twentieth annual symposium on Computational geometry},
  pages={347--356},
  year={2004}
}

@book{chazal2016structure,
  title={The structure and stability of persistence modules},
  author={Chazal, Fr{\'e}d{\'e}ric and De Silva, Vin and Glisse, Marc and Oudot, Steve},
  volume={10},
  year={2016},
  publisher={Springer}
}

@book{carlsson2021topological,
  title={Topological data analysis with applications},
  author={Carlsson, Gunnar and Vejdemo-Johansson, Mikael},
  year={2021},
  publisher={Cambridge University Press}
}

@inproceedings{venna2001neighborhood,
  title={Neighborhood preservation in nonlinear projection methods: An experimental study},
  author={Venna, Jarkko and Kaski, Samuel},
  booktitle={International conference on artificial neural networks},
  pages={485--491},
  year={2001},
  organization={Springer}
}

\appendix
\section{Proofs}
\label{app:proofs}

Here we will provide proofs for the propositions in the main text that were omitted. We first provide a more direct proof of \autoref{prop:marginal}, then we introduce a little bit of poset theory to show how this falls out as a standard result from that area.

\paragraph{Proof of \autoref{prop:marginal}}
Recall that we want to show that
the marginal map \begin{equation}
m: \mathcal{P}^n(U) \to \mathcal{F}^n(U), ~ p \mapsto   \mu_{p}, ~ \mu_p(\sigma) = p(S \geq \mathbf{S}(\sigma)) = p(S(\sigma) = 1)  \end{equation} is surjective.
We have already shown in the main text that $m$ maps probability measures to fuzzy weights, that is 
\begin{equation}
\text{Im}(m) \subset \mathcal{F}^n(U)
\end{equation}
Now consider similarly to $m$, a map $M$ defined on the function space $\mathbb{R}^{\mathcal{S}^n(U)}$, which assigns $M(f)(\sigma) = \sum_{S \geq \mathbf{S}(\sigma)} f(S)$. 
Then $m$ is the restriction of $M$ to the compact, convex subset  $\mathcal{P}^n(U)$, and since $M$ is linear, 
it is also clear that the image of $m$ is a compact convex set, that is in particular  (for $conv$ the closed convex hull)
\begin{equation}
conv(\text{Im}(m)) = \text{Im}(m).
\end{equation}
Furthermore, we note that any simplicial set, that is any element of $\mathcal{S}^n(U)$, is in the image of $m$: 
Let $S$ be a simplicial set, identified here with its weight function. Consider the probability measure given by 
\begin{equation}
p = \delta_{S},
\end{equation}
that is, $p(S') = \delta(S= S')$. Then 
\begin{equation}
  \begin{split}
m(p)(\sigma) &=  \mu_{p} (\sigma) \\&= p[S \geq \mathbf{S}(\sigma)] \\&= \delta(S \geq \mathbf{S}(\sigma)) \\&= \delta(S(\sigma) = 1) \\&= S(\sigma).
\end{split}
\end{equation}
Hence we have 
\begin{equation} 
\mathcal{S}^n(U) \subset \text{Im}(m)
\end{equation}
and so in particular also 
\begin{equation}
conv(\mathcal{S}^n(U)) \subset conv(\text{Im}(m)) = \text{Im}(m).
\end{equation}
By \autoref{lem:extrem} below and the Krein-Milman-theorem (\cite{rudin1991functional}) \begin{equation}
conv(\mathcal{S}^n(U) ) = \mathcal{F}^n(U)
\end{equation}
which concludes the proof.

\begin{lemma}\label{lem:extrem}
The set $\mathcal{F}^n(U)$ is convex, and $\mathcal{S}^n(U)$ are its extremal points.
\end{lemma}
\begin{proof}
Let $S_1,S_2$ be two fuzzy simplicial sets on the same base set $U$, thus identified here with their weight functions.
Then for $\sigma \geq \sigma'$, where the order is face-inclusion, and $t \in [0,1]$
\begin{equation}
t S_1(\sigma) + (1-t) S_2(\sigma) \geq t S_1(\sigma') + (1-t) S_2(\sigma'), 
\end{equation}
hence also $tS_1 + (1-t) S_2$ is monotone, similarly for the degeneracy-order. Thus, fuzzy simplicial sets form a convex set. 
Now let us show that the extremal points are exactly the standard simplicial sets. Recall that the extremal points are those points $S$ where exist no $S_1 \neq S_2, t \in (0,1)$ such that $S = t S_1 + (1-t) S_2$.
It is easy to see that any standard simplicial set is an extremal point, since for any $S_1 \neq S_2$, $t S_1 + (1-t)S_2$ has to take at least one value in $(0,1)$.
Now assume $S$ is an extremal point of $ \mathcal{F}^n(U)$, that is for any $S_1,S_2 \in  \mathcal{F}^n(U)$, $t \in (0,1) $
\begin{equation}
S = t S_1 + (1-t) S_2 \implies S = S_1 = S_2.
\end{equation}
Then $S$ has to take only values in $\{0,1\}$.
Assuming otherwise, by $U$ being finite we may find $\varepsilon$ such that 
\begin{equation}
g^{\pm}(\sigma) =  \begin{cases} a \pm \varepsilon,&\text{if } S(\sigma) = a, a \notin \{0,1\} \\
a,&\text{if } S(\sigma) = a, a \in \{0,1\} \end{cases}
\end{equation}
still fulfills the monotonicity requirements. 
But then $S = \frac{1}{2}(g^+ +g^{-})$, which violates our assumption.
\end{proof}

\subsection{As a standard result on finite posets}

Here we want to show how the above is a special case of a result that will hold generally on finite posets. 
To do so, we will introduce quite an amount of standard terminology and simple lemmas, which will make it easier for the flow of the reader.

\begin{definition}
A partially ordered set (poset) is a set $P$ together with a relation $\leq$ which is  reflexive ($x \leq x$)  antisymmetric $(x \leq y) \land (y \leq x) \implies x =y$ and transitive $x \leq y, y \leq z \implies x \leq z$.
\end{definition}

\begin{example}
The poset we are considering in the main text is that of simplices together with face-inclusion. That is, we may for simplicity ignore degeneracies and have $\sigma \leq \sigma'$ if $\sigma$ may be obtained from $\sigma'$ via face maps.
\end{example}

\begin{definition}
A morphism of posets is a map $f: (P,\leq_P) \to (Q, \leq_Q)$ such that $x \leq_P y \implies f(x) \leq_Q f(y)$. We will call such maps also isotone maps.
We will call a morphism antitone, if it reverses order ($x \leq_P y \implies f(x) \geq_Q f(y)$) (this is just a isotone map under a different order on the codomain, but for clarity it is useful to distinguish).
\end{definition}

\begin{definition}
Let $P$ be a poset. An Up-set is a subset $S \subset P$, such that whenever $x \in S$, $x \leq y$ then also $y \in S$, that is the set is upward-closed. 
A Down-set is defined similar for the order of the inequality reversed.
The Up-set generated by an element $x$ is the smallest Up-set containing $x$, that is,
\begin{equation}
Up(x) = \{y \in P \vert x \leq y\},
\end{equation}
similarly for Down-sets.
\end{definition}

\begin{remark}
By the antisymmetry of the relation, $x = y \iff Up(x) = Up(y)$. Furthermore, if $x \leq y$, then $Up(y) \subset Up(x)$, that is, the map $x \mapsto Up(x)$ is antitone w.r.t. the order of  inclusion on subsets. \end{remark}

\begin{definition}
The indicator function of a set $S$ is 
\begin{equation}
\chi_{S}: P \to \{0,1\}, \chi(x) = \begin{cases} 1, x \in S \\ 0, x \notin S \end{cases}.
\end{equation}
\end{definition}
\begin{remark}
For any set $X$ and a poset $P$, the set of functions $f: X \to P$ is again a poset with $f \leq g \iff f(x) \leq_P g(x) \forall x$.
In particular, $\{0,1\}^P$ forms a poset, which is isomorphic to the powerset of $P$.
Furthermore, $\chi: S \mapsto \chi_{S}$ is an isotone morphism, where the order on sets is set inclusion, as may be readily checked.
\end{remark}
In particular, combining the above remark with the previous one, we obtain:
\begin{lemma}
The map from $P$ to $\{0,1\}^P$, mapping $x \mapsto \chi_{Up(x)}$ is an injective, antitone morphism of posets.
\end{lemma}

\begin{definition}
Let $P$ be a finite or countably infinite poset. We denote by $\mathcal{P}(P)$ the set of all probability mass functions over $P$, that is, functions 
\begin{equation}
\mu: P \to [0,1]
\end{equation}
such that 
\begin{equation}
\sum_{x \in P} \mu(x) = 1.
\end{equation}
\end{definition}

\begin{definition}
Let $P$ be a finite poset. The poset of monotone functions $f: P \to \{0,1\}$ will be denoted as $I(P)$.
The poset of monotone functions $f: P \to [0,1]$ will be denoted as $O(P)$. The latter is called the order-polytope of $P$ - indeed it is a convex polytope of dimension $\vert P \vert$.
\end{definition}
\begin{lemma}
We have the equality
\begin{equation}
I(P) = \{ \chi_{S} \vert S \text{ an Down-set} \}.
\end{equation}
\end{lemma}
\begin{proof}
If $S$ is a down-set, then $\chi_{S}$ is monotone, as $x \leq y$ and $\chi_S(y) =1$ implies $x \in S$ and hence $\chi_S(x) = 1$. Conversely, if $f:P \to \{0,1\}$ is a isotone map and $f(y) = 1$, $x \leq y$ then $f(x) = 1$. 
\end{proof}

\begin{example}
As we have stated in the main text, a classical simplicial set is simply an isotone function from the simplices to $\{0,1\}$, hence alternatively, they may be identified with indicator functions of Down-sets. In particular, the 'minimal simplicial sets' we have described in the main text correspond to $Down(\sigma)$ for a simplex $\sigma$.
Fuzzy simplicial sets then correspond to $O(P)$.
\end{example}

\begin{lemma} (\cite{stanley1986two})
$O(P)$ is convex and $I(P)$ are it's extremal points.
\end{lemma}

\begin{definition}
Let $P$ be a poset. Define the marginal map 
\begin{equation}
m: \mathcal{P}\left (I(P) \right) \to O(P), m(\mu)(x)= \mu([f(x) = 1])  \sum_{f: f(x) = 1} \mu(f) = \sum_{f \in I(P)} f(x) \mu(f) 
\end{equation}
\end{definition}

\begin{theorem}(\cite{caratheodory1911variabilitatsbereich})
Let $X$ be a compact, convex subset set of a finite dimensional topological vector space and $E(X)$ the extremal points of $X$. Then each element $x \in X$ may be written as a convex combination of points in $E(X)$. In other words, each $x$ is the expectation of some probability measure over the extremal points, that is there exists $\mu \in \mathcal{P}(E(X))$ such that 
\begin{equation}
\int_{E(X)} e d\mu(e) = x.
\end{equation}
\end{theorem}

This shows immediately:
\begin{corollary}
The marginal map is surjective. 
\end{corollary}

\begin{remark}
This directly proves \autoref{prop:marginal}
\end{remark}

The other construction we have undertaken is of the following form:
\begin{definition}
Let $P$ be a poset. Define the cumulative-distribution-map (cdm)
\begin{equation}
a: \mathcal{P}(P) \to O(P), a(q)(x) = q([Down(x)]) = \sum_{y \leq x} q(y).
\end{equation}
\end{definition}
\begin{remark}
In our example of simplices, now we are directly taking a probability measure over simplices instead of a probability measure over all classical simplicial set. \end{remark}

The injectivity of this construction is now readily seen by the following result:
\begin{theorem}(Moebius inversion)
On any (finite) poset $C$, we can obtain a Moebius inversion formula (see e.g. \cite{jansma2025mereological}). First we define a Moebius function recursively:
\begin{equation} 
  m(c,c) := 1 ~  \forall c,\quad m(c,d) := - \sum_{c \leq b <d} m(c,b).
\end{equation}
Then for any $f,g : C \to K$, where $K$ is a commutative ring, we have 
\begin{equation}
g(c) = \sum_{c' \leq c} f(c') \iff f(c) = \sum_{c' \leq c} g(c') m(c',c).
\end{equation}
\end{theorem}

\begin{corollary}\label{cor:cdminjective}
The cdm is injective, as the Moebius inversion formula provides an explicit inverse.
\end{corollary}

Hence, this in particular proves \autoref{prop:injection} as a special case.

\section{Datasets}

\paragraph{MNIST}
The MNIST  \citep{lecun1998mnist} dataset  consists of 70,000 grayscale images of handwritten digits (0–9), each of size 
28×28 pixels. We use the standard split of 60,000 training and 10,000 test images, with pixel intensities rescaled to 
$[0,1]$ and flattened into 784-dimensional vectors.

\paragraph{Fashion-MNIST}
Fashion-MNIST \citep{xiao2017fashion} contains 70,000 grayscale images of clothing items from 10 classes (e.g. t-shirts, trousers, shoes), with the same 
28×28 format as MNIST. We use the canonical split of 60,000 training and 10,000 test images, normalized to 
$[0,1]$ and flattened to 784-dimensional vectors.

\paragraph{pbmc68k}
The PBMC 68k dataset  \citep{zheng2017massively} is a single-cell RNA-seq dataset of 
around 68,000 peripheral blood mononuclear cells from a single donor, originally released as a 10x Genomics demonstration dataset and distributed via Scanpy. We use the 50-dimensional PCA embedding provided by Scanpy together with unsupervised cluster assignments (Louvain communities) as cell-type–like labels.

\paragraph{Hematopoiesis}
The hematopoiesis dataset \citep{paul2015transcriptional} is a single-cell RNA-seq dataset of murine bone marrow cells covering multiple stages of myeloid differentiation. Following the standard Scanpy preprocessing pipeline, we normalize counts, log-transform, select highly variable genes, compute a PCA embedding, and use graph-based clustering (Louvain) to obtain discrete cell-state labels.

\paragraph{USPS}
The USPS dataset \citep{hull2002database} is a handwritten digit recognition benchmark collected from U.S. postal mail. It contains 9,298 grayscale images of digits (0–9), each of size 
16×16 pixels. We use the OpenML version (usps, version 2), rescaling the pixel intensities to 
$[0,1]$.

\paragraph{COIL-20}
COIL-20 \citep{Nene1996coil} is an object recognition dataset comprising 20 household objects imaged at 72 different viewpoints around a 360° rotation, yielding 1,440 grayscale images. We download the public COIL-20 archives from Columbia University, convert each image to grayscale if needed, resize to 
128×128 pixels, normalize intensities to 
$[0,1]$ and assign labels corresponding to the underlying object identity.

\paragraph{20news} 
The 20 Newsgroups dataset \citep{Mitchell1997Twenty} is a text classification benchmark of roughly 18,000 Usenet posts partitioned into 20 topical categories (e.g. politics, sports, science). We use the fetch\_20newsgroups version from scikit-learn, remove headers, footers, and quotes, and represent documents using TF–IDF features; we apply truncated SVD followed by 
normalization to obtain dense low-dimensional embeddings.

\paragraph{CIFAR-10}
CIFAR-10 \citep{Krizhevsky09learningmultiple} consists of 60,000 color images of natural objects in 10 classes (airplane, automobile, bird, etc.), with resolution 
32×32 pixels. We use the standard split of 50,000 training and 10,000 test images; images are resized to 
224×224, normalized with ImageNet statistics, and passed through a ResNet-18 pretrained on ImageNet \citep{deng2009imagenet} to obtain 512-dimensional feature vectors from the penultimate layer.

\section{On appropriate weight-to-distance functions}
\label{app:OnAppropriateWeight-to-distanceFunctions}

In this appendix, we show that a probabilistic perspective naturally arises when transferring metrics to fuzzy weights.  Here, we use the categorical definitions of the respective objects, consult \cite{barth2024fuzzysimplicialsetsapplication} for more details on this.
We recall the following facts on fuzzy (simplicial) sets: 
First, remember that by $I$ we denote $[0,1]$ as a topological space with the inclusion maps $i_{ab}: [0,a) \to [0,b)$.
\begin{definition}
A fuzzy set is a sheaf $S: I \to \textbf{Sets}$ where the restriction maps $S(i_{ab}): S(a) \to S(b)$ are injections.
\end{definition}
\begin{remark}
The sheaf condition in this case translates to 
\begin{equation}
\lim_{b \in B} S(b) \simeq S(a)
\end{equation}
whenever $\sup B = a$.
\end{remark}

\begin{definition}
A classical fuzzy set is a set $X$ together with a membership function $\mu: X \to [0,1]$.
\end{definition}
\begin{proposition}
There is an isomorphism of classical fuzzy sets and fuzzy  sets:
Map a fuzzy set $S$ to $(S(0), \eta)$ where $\eta(s) = \sup \{ a: s \in S(a))\}$.
Map a classical fuzzy set $(X,\mu)$ to  $S$, where $S(a) = \mu^{-1}([a,1])$.
\end{proposition}
Note that we need the condition $s \in S(\eta(s))$ - which is given by the sheaf-condition- for this construction to work, which ensures that we have invertibility in the sense that $s \in \mu^{-1}[\eta(s),1]$. 

Now furthermore recall that given a weight-to-distance function $\phi$, which we will define below, the singular set functor is defined as
\begin{definition}
The singular set functor $\text{Sing}_\phi$ maps an (uber) metric space to the fuzzy simplicial set  $S$ where 
\begin{equation}
S(n,a) = \{(r_0,...,r_n) \in X \times ... \times X \vert d(r_i,r_j) \leq \phi(a) \forall i,j \}.
\end{equation}
\end{definition}

For this definition to yield a valid fuzzy (simplicial) set, we have some natural restrictions on $\phi$. 
First, by the injectivity of the inclusion maps we need that $S(n,a) \subset S(n,b)$ for $b\leq a$. This means that $\phi$ should be a decreasing function. 
Second, if we want that always $S(n,0) = X \times ... \times X$ (n +1 times), then we need $\phi(0) = \infty$. 
Third, if we want that $S(n,1)$ only consists of degenerate simplices (tuples with only the same point repeated), then we need $\phi(1) = 0.$  
Lastly, we need the sheaf condition to be fulfilled. That is, we need for any sequence 
$b_k \rightarrow a$, $b_k \leq a$, if $(r_0,...,r_n) \in S(n,b_k) \forall k$ then also $(r_0,...,r_n) \in S(n,a)$. 
To make this hold for all possible metrics, we thus need 
\begin{equation}
c \leq \phi(b_k) \forall k \implies c \leq \phi(\lim_{k \to \infty} b_k), \forall c.
\end{equation}
and therefore the limit from below:
\begin{equation}
\lim_{b_k \rightarrow a{-}} \phi (b_k) = \phi(a).
\end{equation}

Combining these properties then yields an adequate definition for a weight-to-distance function.
\begin{definition}
A weight-to-distance function is a function $\phi: [0,1] \to [0,\infty]$ which is decreasing, left-continuous and for which 
\begin{itemize}
\item $\lim_{b_k \to 0} \phi(b_k) = \infty$
\item $\lim_{b_k \to 1} \phi(b_k) = 0$
\end{itemize}.
\end{definition}

We now want to study what these properties of a weight-to-distance function imply about the inverse, i.e., a distance-to-weight function. 
To do so, we will need the notion of a generalized inverse of a monotone function.
The following result is well known \citep{embrechts2013note}.
\begin{proposition}
Let $f: \mathbb{R} \to \mathbb{R}$ be an increasing function. Then, the generalized inverse $f^{-}: \mathbb{R} \to [-\infty, \infty]$ is defined as 
\begin{equation}
f^{-}(y) = \inf \{x \vert f(x) \geq y\} = \inf f^{-}([y,\infty)).
\end{equation}
Then $f^{-}$ has the following properties
\begin{itemize}
\item $f^{-}$ is an increasing function
\item At any point $y$ where $f^{-}(y)$ is finite, $f^{-}$ is left-continuous.
\item $f^{-}(f(x)) \leq x$ and for an injective function we have equality
\item $f(x) \geq y \implies x \geq f^{-}(y)$ and if $f$ is a right continuous function, then  $ x \geq f^{-}(y) \implies f(x) \geq y$.
\end{itemize}
\end{proposition}

We want to transfer these properties to the setting of decreasing functions, which we are concerned with - this should also be well known, but we prove it here for convenience.
First, some auxiliary definition
\begin{definition}
We define the negation and flip operator $U_c$, which operates on functions $g$ as 
$U_c(g)(y) = c-g(c-y)$. We denote $U_0 =:U$.
\end{definition}
\begin{lemma}
$U_c$ maps decreasing/increasing functions to increasing/decreasing functions, and left/right continuous functions to right/left continuous functions.
\end{lemma}
\begin{proof} Obvious from the definition.
\end{proof}
Now, we define the generalized inverse of a decreasing function as follows:

\begin{definition}
For a decreasing function $f$, define 
\begin{equation}
f^{--}(y) = \sup \{x \vert f(x) \geq y\} = \sup f^{-1}([y,\infty)).
\end{equation}
\end{definition}
This definition relates to the generalized inverse of an increasing function in the following way:
\begin{lemma}
We have, for any $c$ such that $x\mapsto c-x$ is a bijection of the domain of $f$,
\begin{equation}
f^{--}(y) =  U_c((U_cf)^{-}) (y)
\end{equation}
\end{lemma}
\begin{proof}
\begin{equation} \begin{split}
f^{--}(y) &= \sup \{x \vert f(x) \geq y\} \\
&=\sup \{x \vert -f(x) \leq -y \} \\
& = \sup \{c- u \vert -f(c-u) \leq -y\} \\
& =c - \inf \{u \vert -f(c-u) \leq -y \} \\
& =c - \inf \{u \vert c-f(c-u) \leq c-y \} \\
& = c - U_c(f^{-})(c-y) \\
&=  U_c ((U_cf)^{-}) (y).
\end{split} \end{equation}
Below, we assume that $c$ is chosen appropriately (e.g. for a function with domain $[0,1]$, c = 1).
\end{proof}
Through this relationship, we get the following properties of the generalized inverse of a decreasing function from that of an increasing one (for each point, compare with the corresponding point in the proposition above).
\begin{corollary}
$f^{--}$ has the following properties
\begin{enumerate}
\item $f^{--}$ is a decreasing function
\item At any point $y$ where $f^{--}(y)$ is finite, $f^{--}$ is right-continuous.
\item $f^{--}(f(x)) \geq x$ and for an injective function we have equality
\item $f(x) \leq y \implies x \leq f^{--}(y)$ and if $f$ is a left continuous function, then  $ x \leq f^{--}(y) \implies f(x) \leq y$.
\end{enumerate}
\end{corollary}
\begin{proof}
\begin{enumerate}
\item $Uf$ is an increasing function, so is $(Uf)^{-}$, and hence $U(Uf)^{-}$ is decreasing.
\item If at $y$, $U((Uf)^{-})$ is finite, so is $(Uf)^{-}$, which is hence left continuous and hence $U((Uf)^{-})$ is right continuous.
\item $c+ U((Uf)^{-})(f(x)) = (Uf)^{-}(-f(x)) = (Uf)^{-} (Uf(-x))$, now we have $(Uf)^{-} (Uf(-x)) \leq - x$ from the properties of $(Uf)^{-}$ and the claim follows.
\item For brevity, we only prove the iff case where $f$ is left continuous:
\begin{equation} \begin{split}
x \leq f^{--}(y) &\iff  x\leq -(Uf)^{-}(-y) \\ &\iff  -x \geq (Uf)^{-}(-y)
\\ &\iff  Uf(-x) \geq -y  \\ &\iff -f(-(-x)) \geq -y  \\& \iff f(x) \leq y
\end{split} \end{equation}
\end{enumerate}
\end{proof}

Thus, we now know how the properties of our weight-to-distance function transfer to properties of the distance-to-weight function. In particular we note that left-continuity is turned into right-continuity.
\begin{corollary}
For any weight-to-distance function $\phi$, the generalized inverse $\phi^{--}$ fulfills:
\begin{itemize}
\item $\phi^{--}(0) = 1$
\item $\lim_{w_k \to \infty} \phi^{--}(w_k) = 0$.
\item $\phi^{--}$ is decreasing and right-continuous.
\end{itemize}
Thus, $1-\phi^{--}$ is a cumulative distribution function in the sense of probability theory (or, $\phi^{--}$ is a survival function).
\end{corollary}

Thus, we will call the generalized inverse $\phi^{--}$ of a weight-to-distance function a distance-to-weight function. We observe that we obtain the same relation to a cumuluative distribution function as in our example of Vietoris-Rips complexes at random scales sampled from some distribution.
Thus, this construction naturally arises when transferring between fuzzy simplicial sets and metric spaces.

\end{document}